\documentclass[lettersize,journal]{IEEEtran}
\usepackage{amsmath,amssymb,amsfonts}
\usepackage{algorithmic}
\usepackage{algorithm}
\usepackage{array}
\usepackage[caption=false,font=normalsize,labelfont=sf,textfont=sf]{subfig}
\usepackage{textcomp}
\usepackage{stfloats}
\usepackage{url}
\usepackage{verbatim}
\usepackage{graphicx}
\usepackage{cite}
\usepackage[english]{babel}
\usepackage{amsthm}
\newtheorem{definition}{Definition}
\newtheorem{theorem}{Theorem} 
\newtheorem{assumption}{Assumption} 
\newtheorem{lemma}{Lemma}

\IEEEoverridecommandlockouts
\usepackage{amsmath}

\begin{document}

\title{Enhancing Federated Learning with Adaptive Differential Privacy and Priority-Based Aggregation}

\author{\IEEEauthorblockN{1\textsuperscript{st} Mahtab Talaei\IEEEauthorrefmark{1}}\\
\IEEEauthorblockA{\textit{Department of Electrical and Computer Engineering} \\
\textit{Isfahan University of Technology}\\
Isfahan, Iran \\
mtalaei@bu.edu} \\
\and
\IEEEauthorblockN{2\textsuperscript{nd} Iman Izadi}\\
\IEEEauthorblockA{\textit{Department of Electrical and Computer Engineering} \\
\textit{Isfahan University of Technology}\\
Isfahan, Iran \\
iman.izadi@iut.ac.ir}
}

\maketitle

\begin{abstract}
Federated learning (FL), a novel branch of distributed machine learning (ML), develops global models through a private procedure without direct access to local datasets. However, it is still possible to access the model updates (gradient updates of deep neural networks) transferred between clients and servers, potentially revealing sensitive local information to adversaries using model inversion attacks. Differential privacy (DP) offers a promising approach to addressing this issue by adding noise to the parameters. On the other hand, heterogeneities in data structure, storage, communication, and computational capabilities of devices can cause convergence problems and delays in developing the global model. A personalized weighted averaging of local parameters based on the resources of each device can yield a better aggregated model in each round. In this paper, to efficiently preserve privacy, we propose a personalized DP framework that injects noise based on clients' relative impact factors and aggregates parameters while considering heterogeneities and adjusting properties. To fulfill the DP requirements, we first analyze the convergence boundary of the FL algorithm when impact factors are personalized and fixed throughout the learning process. We then further study the convergence property considering time-varying (adaptive) impact factors.

\end{abstract}


\footnotetext[1]{Mahtab Talaei was affiliated with the Department of Electrical and Computer Engineering at Isfahan University of Technology during the research for this paper. At the time of submission, she is affiliated with the Division of Systems Engineering at Boston University, Boston, USA.}

\section{Introduction}
Smart distributed systems such as smartphones, automated vehicles, multi-agent systems, and wearable devices are growing rapidly in our daily lives. Their underlying mechanism which is attached with sensing and communicating generates an unprecedented amount of data every day. Therefore, utilizing these sources of rich information to enhance services offered to people and organizations owning the data, without violating their privacy, matters a great deal. The developments in the computational and communicational capabilities of intelligent distributed devices along with their abilities to collect and store large datasets have opened up effective alternatives for managing and analyzing local databases.

A common traditional practice to develop predictive machine learning (ML) models is to transmit raw data over networks and generate models in a centralized manner. While this method has provided data owners valuable services throughout the years, their efficiency for today's crowdsourced data is called into question. Communication costs of sending large volumes of data on one hand, and privacy concerns for sharing personal information on the other hand have provided space for decentralized ML algorithms, such as federated learning (FL)~\cite{ref1}. 

Federated machine learning is a promising solution in settings dealing with large volumes of data as well as privacy concerns about clients' sensitive information~\cite{ref8}. In this framework, each device builds its model using local datasets, and the essential model parameters, rather than raw data, are transmitted to the cloud server. The server aggregates these parameters and updates the global model throughout a recursive downloading and uploading cycle~\cite{ref2,ref3,ref4}. Hence, each client benefits from a larger database during the learning process, without direct access to it. While offering great advantages over conventional ML methods, FL has its own challenges. Expensive communications, systems and statistical heterogeneities, and security risks are considered as the four main issues while developing FL models~\cite{ref5}.

\IEEEpubidadjcol

Deep Learning (DL) models are widely used in FL, especially for feature extraction in the large image, voice, and text datasets. In order to optimize local DL models inside the clients, stochastic gradient descent (SGD) is generally adopted~\cite{ref6,ref7,talaei2024comments}. Sending frequent gradient updates with the massive number of both parameters and clients in FL leads to an extreme rise in communication costs. Increasing the number of local updates~\cite{ref9, ref10} is one natural way to modify communication bottlenecks with more local computations. On the other hand, quantization~\cite{ref11, ref12, ref15} and sparsification~\cite{ref13, ref14} methods mitigate this challenge by reducing the size of transmitted messages in each round. 

Dealing with systems that have different computational capabilities, network capacities, and power resources is an inevitable challenge in FL. Several approaches, including resource-based client selection~\cite{ref16, ref17}, robust and fault-tolerant algorithms~\cite{ref18}, and asynchronous communications~\cite{ref19} address these challenges. On the other extreme, heterogeneity in data distributions of the clients causes problems in the training and convergence of FL algorithms. Using multi-task learning methods~\cite{ref26, ref28} and avoiding local minimums by adding a proximal term to the objective function~\cite{ref27} help handling unbalanced and non-IID data in FL~\cite{ref5}.

Even though the idea of FL was first proposed for its strong privacy guarantees, it has been shown that local datasets can be still revealed to stragglers using model inversion attacks on shared updates~\cite{ref29}, especially when DL is used in local models~\cite{ref30, ref31}. To mitigate this challenge, differential privacy (DP) is one of the widely used protection algorithms due to its solid theoretical guarantees~\cite{ref20,talaei2024adaptive}. In order to reduce the risk of data leakage in ML algorithms, noise with Gaussian, Laplace, or Exponential distribution is deliberately added to data in DP. The work in~\cite{ref23} proposes a global DP algorithm in FL and gives a theoretical explanation for the convergence behavior of the suggested scheme.

As discussed earlier, nodes in a distributed architecture differ in the data structure, dataset size, network condition, reliability, availability, and computation capabilities, which can even be time-varying. A privacy-preserving approach in FL is not effective unless paying attention to these personalized characteristics. Hence, there exist multiple works on DP with the content ``adaptive" to compensate systems and statistical heterogeneities in FL. These works can be divided into two general directions based on the adaptability criterion. One direction injects an adaptive noise distribution to local parameters to enhance local protection. It considers each client separately without involving their heterogeneity. For instance, adaptive clipping ~\cite{ref34} finds the best clipping constant for DP in each device based on their local behaviors. In~\cite{ref32}, noise with Laplace distribution is added to model updates based on the neurons' contributions in the clients. The work in~\cite{ref33} achieves a trade-off between privacy and accuracy by adding more noise to less important parameters and less noise to more important ones. 

The second direction, however, concentrates on personalized training in the heterogeneous networks. The work in~\cite{ref35} trains differentially private models in each client and uploads the local updates for the server. These directions both lack considering the local characteristics in the aggregating process. More specifically, they assume the same impact factor for all devices during aggregation, regardless of their local dissimilarities. This assumption not only simplifies the convergence analysis of the algorithm but also changes the DP requirements~\cite{ref7}. To the best of our knowledge, the privacy and convergence analysis in FL with non-identical and time-varying impact factors have not yet been studied in the existing literature. 

In this paper, we combine the heterogeneity and privacy concerns in a novel FL scheme. Regardless of multi-task learning algorithms used in FL, each local model possesses a weight or an impact in the global cost function. This impact can be assigned considering many factors by the server or the clients. It can also change (increase or decrease) or even become zero during the learning process. We, therefore, propose a DP algorithm considering the non-identical impact factors, namely, personalized aggregation in differentially private federated learning (PADPFL). We further establish the convergence analysis of the algorithm and the influence of the additive noise on it.

In summary, the main contributions of this paper are as
follows.

\begin{itemize}
\item{We propose a noise injection paradigm, PADPFL, that satisfies DP requirements with Gaussian distribution when clients have different impact factors in the aggregation process.}
\item{We perform a convergence analysis of the proposed algorithm for Non-IID clients when using fixed non-identical impact factors throughout training the global model.}
\item{We perform a convergence analysis of the proposed algorithm for Non-IID clients when using adaptive (time-varying) impact factors throughout training the global model.}
\item{We conduct evaluations on real-world datasets to verify the effectiveness of PADPFL,
and observe the trade-off between model accuracy, privacy budget, and impact factors.}

\end{itemize}

The remainder of this paper is organized as follows. In
Section II, we review some preliminaries on FL,and DP. In Section III, we introduce our approach for a differentially private federated learning in a client and server side. Next, we analyse the convergence bound on the global loss function of the proposed solution for a fixed and time-varying impacts, in Section IV and V, respectively. Simulations and results are
presented in Section IV, and the summary and conclusion are
given in Section VI.

\section{Preliminaries}
In this section, we briefly review some key materials of FL and DP.

\subsection{Federated Learning}
The goal in a standard FL problem is to develop a global ML model for tens to millions of clients without direct access to their local datasets~\cite{ref21}. The only messages transmitted from the clients to the cloud server in this framework are the training parameter updates of the local loss functions. To formalize this goal, consider $N$ clients as depicted in Fig.~\ref{fig:FL}. We wish to find weight matrix $x$ minimizing the following loss function:  
\begin{equation}
\min_x L(x),\, \textnormal{where} \, L(x):= \sum _{i=1} ^N p_i l_i(x, \mathcal{D}_i), 
\label{eq1}
\end{equation}
where $l_i$ and $D_i$ represent the local loss function and training database of the $i$-th client, respectively. Moreover, the coefficient $p_i$ is considered as the relative impact factor of device $i$ in the global model, so that
\begin{equation}
\sum _{i=1} ^ N p_i = 1, \qquad 0\leqslant p_i \leqslant 1 \label{eq2}
\end{equation}

In order to solve \eqref{eq1}, matrix of the global parameters at itteration $(t)$ is updated using weighted averaging of the trained local parameters ($x_1 ^{(t)}, x_2 ^{t}) , ... x_N^{(t)}$)~\cite{ref22}
\begin{equation}
x^{(t)}:= \sum _{i=1} ^N p_i x_i^{(t)}, \label{eq3}
\end{equation}

\begin{figure}[b]
\begin{center}
\includegraphics[width=1\linewidth]{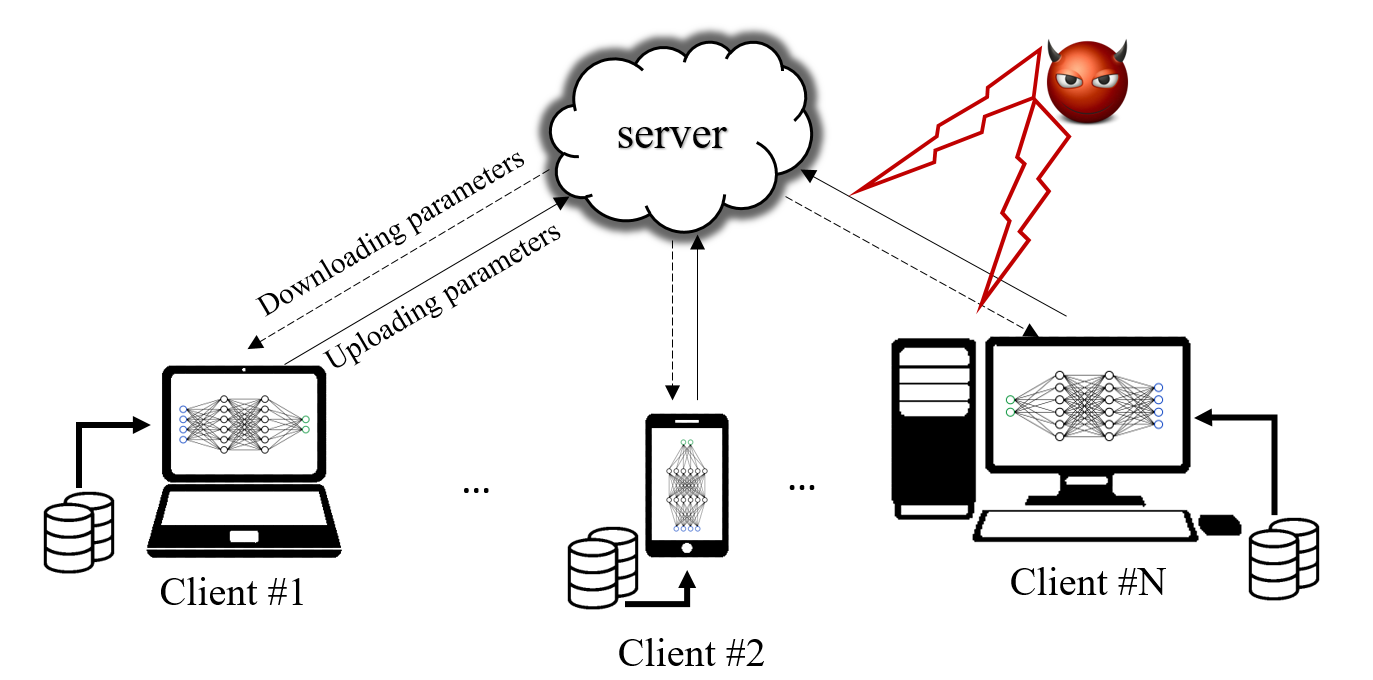}
\caption{A FL training model. }\label{fig:FL}
\end{center}
\end{figure}

To address heterogenities, FedProx~\cite{ref27} is utilized in the learning process. Therefore, defining
\begin{equation}
h_i(x_i^{(t+1)};x^{(t)}) = l_i(x_i^{(t+1)}) +
\frac{\mu}{2} \lVert x_i^{(t+1)} - x^{(t)} \rVert ^2,  \quad \gamma \in [0,1],
\label{eq4} \end{equation}
$x_0$ is the $\gamma_i$-inexact solution for $\min_x h_i(x,x^{(t)})$.

\subsection{Differential Privacy}

DP gives a rigorous mathematical definition of privacy and strongly guarantees preserving data in ML algorithms.
A randomized mechanism $\mathcal{M}$ is differentially private if its output is robust to any change of one sample in the original dataset. The following definition formally clarifies this statement for $(\epsilon, \delta)$-DP \cite{ref20}:

\begin{definition}[$(\epsilon, \delta)$-DP]
A randomized mechanism $\mathcal{M}: \mathcal{X}\rightarrow\mathcal{R}$ satisfies $(\epsilon, \delta)$-differential privacy for two non-negative numbers 
$\epsilon$ and $\delta$ if for all adjacent datasets $\mathcal{D}$ and $\mathcal{D}^{\prime}$ $d(\mathcal{D},\mathcal{D}^{\prime}) = 1$, and for all subsets $S \subseteq
\mathcal{R}$, there holds
\begin{equation}
\textnormal{Pr}[\mathcal{M}(\mathcal{D}) \in S] 
\leqslant 
e ^{\epsilon} \textnormal{Pr}[\mathcal{M}(\mathcal{D}^{\prime}) \in S]
+ \delta,
\end{equation}
\end{definition}
\noindent where the randomized algorithm $\mathcal{M}$ maps an input $x \in \mathcal{X}$ discretely to $\mathcal{M}(x)=y$ with probability $(\mathcal{M}(x))_y,\, \forall y \in \mathcal{R}$. The probability space is defined over the coin flips of the mechanism $\mathcal{M}$. Note that the difference between two datasets $\mathcal{D}$ and $\mathcal{D}^{\prime}$, $d(\mathcal{D},\mathcal{D}^{\prime})$, is typically defined as the number of records on which they differ.

It is concluded from this definition that, with a probability of $\delta$, the output of a differentially private mechanism on two adjacent datasets varies more than a factor of $e ^\epsilon$. Thus, smaller values of $\delta$ enhance the probability of having the same outputs. Smaller values of $\epsilon$ narrow down the privacy protection bound. The smaller $\epsilon$ and $\delta$, the lower the risk of privacy violation. 

Based on \cite{ref20}, considering $f$ as an arbitrary $d$-dimensional function applied on a dataset, for $\epsilon \in (0,1)$ and $c \geqslant \sqrt{2 \ln{(1.25/ \delta)}}$, a Gaussian mechanism with parameter $\sigma \geqslant c \Delta f / \epsilon $ that deliberately adds Gaussian noise scaled to $\mathcal{N}(0,\sigma ^ 2)$ to each output component of $f$ is $(\epsilon, \delta)$-differentially private. Here, $\Delta f$ is the sensitivity of the function $f$ defined by $\Delta f = \max _{\mathcal{D}, \mathcal{D}^{\prime}} \lVert f(\mathcal{D}) - f({\mathcal{D}^{\prime}}) \rVert $.

\section{Personalized Differential Privacy in Federated Learning}
In this section, we propose the personalized noise injection for preserving DP. We first describe the threat model and then propose the algorithm.  

\subsection{Threat Model and Design Goals}
We consider the cloud server to be an ``honest-but-curious" entity, ie, the central server can use model inversion attacks to recover training data. Additionally, local and global parameter updates can be revealed to adversaries in the uploading and downloading channels. For this reason, the goal of our approach is to protect the weights transmitted between the server and clients from being inferred any extra information about users to both the server or external adversaries. Preserving global privacy is the primal goal of our approach, but it leads to a level of local privacy, as well.
 
Following from \cite{ref23}, we also assume that downloading channels are exposed to more external attacks than uploading channels as they are broadcasting. Hence, considering $T$ aggregation times, the revelation of local parameters while uploading can be at most $R$ times ($R \leqslant T$).

\subsection{Proposed Privacy-Preserving Scheme}
considering systems and statistical heterogeneities, each client influences the global loss function ($L$) differently. Apart from multi-task learning methods that aim to develop personalized local models, the aforementioned impact factor ($p_i$) can play a vital role in training accurate models. The impact factors assigned to the clients at each iteration can strengthen or weaken the effect of local models in the global loss function.  

Although using non-identical impact factors may seem a straightforward approach, its importance is underestimated in the literature. Assuming the natural setting $p_i=\frac{1}{N}$ or even $p_i= \frac{m_i}{m}$, where $m_i=\vert\mathcal{D}_i\vert$ and $m= \sum_i m_i$ is the total number of samples, is far from reality and oversimplifies the problem. In fact, clients participate in learning in different ways as their data and structures are not the same. To compensate for these heterogeneities, we assign different impacts to clients while aggregating models. The heterogeneities between the clients can come from several sources, including:
\begin{itemize}
\item{data quality}
\item{data reliability}
\item{dataset size}
\item{link quality}
\item{revelation probability}
\item{accessibility}
\item{client reliability}
\end{itemize}
The first three items relate to the clients, while the remaining depends primarily on the knowledge of the central server. 

While working collaboratively with possibly millions of users, even when datasets have the same nature, the quality of information used in local models is not the same. For instance, while modelling image datasets, the resolution of data varies between the clients, and hence, training local models based on low-quality images reduces the global model performance. Moreover, the reliability of local datasets can influence the validity of models, as they may contain irrelevant information. So, users can send additional bits to the cloud server, based on local model performances, to help the server assign relevant impact factors. On the other hand, the size of local datasets does not necessarily affect impact factors directly. In Non-IID structures, we may have valuable informative datasets that are relatively small in size, but global models should be biased in favor of them. Therefore, the assumption $p_i= \frac{m_i}{m}$ would not be a wise choice.

Stemming from the fact that distributed learning requires training and inference of models over a wireless system, uncertainty and stochasticity exist in its nature. Link errors and delays can adversely influence the convergence speed of the learning algorithm \cite{ref24}. However, utilizing variant impact factors can mitigate this challenge to some extent. When client $k$ cannot synchronize itself with the others due to network faults, the server can distribute $p_k$ among counterpart clients for that iteration to keep pace with the algorithm. Additionally, different levels of accessibility and reliability of local parameters lead to utilizing non-identical impact factors. When several groups of IoT devices and sensor networks perform the measurements and model updates cooperatively in a FL task, weighting the updates appropriately and based on the devices' accuracy, reliability, or accessibility must be a priority for the server.

Along with all the aforementioned heterogeneity sources, impact factors can vary between global iterations. In other terms, the impacts assigned to the local model weights are not fixed throughout the learning process. They can change to accommodate different situations. For instance, a client sending accurate updates may become out of charge or encounter noisy links in the middle of learning process. Hence, a wise server should adaptively reduce the impact factor of its parameters to save the global model performance.

In this paper, we achieve $(\epsilon, \delta)$-DP using Gaussian mechanism, which provides theoretical privacy guarantees for sharing DL model updates. Here, we calculate the amount of the client-side and server-side additive noise based on the sensitivity parameter when impact factors are non-identical. The differential privacy requirements are satisfied for each iteration, and the only limitation on impact factors is  \eqref{eq2}.

\subsubsection*{\bf{Client-side DP}}
Assume that local model parameters are sent to the cloud server as the updates. Setting the batch size equal to the local dataset size
$\lvert \mathcal{D}_i \rvert = m_i$, the function $f$ to be protected is defined as
\begin{equation}
\begin{split}
f_i (\mathcal{D}_i) \triangleq 
x_i &= \text{arg} \min _x l_i(x, \mathcal{D}_i)\\ 
&= \frac{1}{m_i} 
\sum _{j=1} ^ {m_i} \text{arg} \min _x l_i(x, \mathcal{D}_{i,j}), 
\quad \forall i
\end{split}
\end{equation}

\noindent By clipping the local weights using a bounding limit $B$, $ \lVert x_i \rVert \leq B$ \cite{ref7}, the sensitivity of $f_i$ is calculated as
\begin{equation}
\begin{split}
\Delta f_i &= \max _{\mathcal{D}_i , \mathcal{D}^{\prime}_i} 
\lVert f_{i ,\mathcal{D}_i} - f_{i, \mathcal{D}^{\prime}_i} \rVert \\
&= \max _{\mathcal{D}_i , \mathcal{D}^{\prime}_i}
\bigg \lVert \frac{1}{m_i} \bigg (
\sum _{j=1} ^ {m_i} \text{arg} \min _x l_i(x, \mathcal{D}_{i,j}) \\
&- \sum _{j=1} ^ {m_i} \text{arg} \min _x l_i(x, \mathcal{D}^{\prime} _{i,j}) \bigg) \bigg\rVert \\
&= \frac{1}{m_i} \max \big \lVert 
\text{arg} \min _x l_i(x, \mathcal{D}_{i,k}) -
\text{arg} \min _x l_i(x, \mathcal{D}^{\prime} _{i,k}) \big \rVert \\
&= \frac{2B}{m_i}, 
\label{eq5}
\end{split}
\end{equation}

\noindent where, based on the definition of sensitivity here, the $i$-th client's dataset $\mathcal{D}_i$ and $\mathcal{D}^{\prime}_i$ differ only in one sample ($k$-th sample).

To ensure $(\epsilon,\delta)$-DP for each client in FL, we have to add Gaussian noise with parameter $\sigma = c \frac{2B}{m_i \epsilon}$ to the weight matrices of all clients before uploading. Considering the maximum revelation times $R$, $\sigma$ should be multiplied with it to guarantee the desired protection level of local parameters. Hence, to have a united noise parameter, we define the standard deviation (SD) of the additive noise in the client-side as
\begin{equation}
\sigma _{C_i} = \frac{2B R \, c}{\min \{m_i\} \, \epsilon}, \quad \forall i
\label{eq6}
\end{equation}

\subsubsection*{\bf{Server-side DP}} 
The function to be protected in the server side is the global aggregated weight transmitted to the clients defined by
\begin{equation}
f \triangleq x = \sum _ {i=1} ^ N p_i x_i.  
\label{eq7}
\end{equation}
Based on the analysis provided in \cite{ref7} and the client-side sensitivity in \eqref{eq5}, the sensitivity of $f$ is bounded as
\begin{equation}
\Delta f  \leq 2 B \frac{\max \{ p_i\}}{\min \{m_i\}}.
\label{eq8}
\end{equation}

Here, to find the SD of the additive Gaussian noise in the server, we first calculate the distribution of the aggregated local noises. The aggregated noisy weight at each iteration is given as
\begin{equation}
x =\sum _{i=1} ^n p_i \tilde{x}_i = \sum _{i=1}^N p_i (x_i + n_i)= 
\sum _{i=1}^N p_i x_i + \underbrace{\sum _{i=1}^N p_i n_i}_{n}, 
\label{eq9}
\end{equation}
where, for the independent normally distributed $n_i \sim \mathcal{N}(0, \sigma _ {C_i}^2)$, we have
\begin{equation}
n \sim \mathcal{N}(0, \underbrace{\sum _{i=1} ^N {\sigma _ {C_i}^2 p_i ^2 }} _ { \sigma ^2 _{AC}}).
\label{eq10}
\end{equation} 
Therefore, we have the following theorem to ensure $(\epsilon,\delta)$-DP from the server perspective.

\begin{theorem}[server-side DP]
Considering $T$ as the aggregation times and the maximum revelations in the broadcasting channels, the SD of the server-side noise is given by
\begin{equation}
\sigma _S = 
\begin{cases}
\frac{2B c \sqrt{ T^2 \max{p_i}^2 - R^2 \sum_{i=1} ^N p_i ^2}}{\min \{m_i\} \, \epsilon}, &{\text{if}}\ T > \frac{R \sqrt{\sum _i p_i ^2}}{\max \{p_i\}}  \\ 
{0,}&{\text{otherwise.}} 
\end{cases}
\label{eq11}
\end{equation}
\end{theorem}
\begin{proof}
The standard deviation of the total desired noise, based on \eqref{eq8}, is $\sigma _ A = \frac{2B \, T \, c \max \{ p_i\}}{\min \{m_i\} \, \epsilon}$. Hence, the variance of the server-side Gaussian noise is calculated by 
\begin{equation}
\sigma ^2 _S = \sigma ^2 _{A} -  \sigma ^2 _{AC},
\label{eq12}
\end{equation}
which results in \eqref{eq11}.
\end{proof}

Applying the client and server-side noise with the calculated Gaussian distributions satisfies $(\epsilon, \delta)$-DP theoretically in the uploading and downloading channels for each iteration. Since the involved clients add noise to the local parameters before uploading for the server, a level of local privacy is also achieved here. The server, subsequently, chooses the relative impact factors based on the information acquired and updates the global parameters.  Then, it decides on the extra server-side noise, $n_s \sim \mathcal{N}(0, \sigma_S)$, and transmits $\tilde{x}= x + n_s$ for the upcoming training cycle.

\section{Convergence Analysis of the \\ Personalized DP in FL}
In this section, we analyze the convergence properties of the proposed algorithm for personalized DP in FL. Our main purpose is to reach a convergence upper limit for the algorithm when we have personalized impact factors.
The required assumptions for our analysis about the properties of the global and local loss functions, regarding their relation $L(x) = \sum _{i=1} ^N p_i l_i (x)$, are as follows:

\begin{assumption}{\ \\}
\vspace{-10pt}
\begin{enumerate}
\item{$l_i(x)$ is convex.}
\item{$l_i(x)$ is $\rho$-Lipschitz smooth, i.e., 
$\lVert \nabla l_i (a) - \nabla l_i (b) \rVert \leqslant \rho \lVert a - b \rVert,\, \forall a, b$. }
\item{$L(x^{(0)}) - L (x ^ *) = \Theta$; where $x^{(0)}$ and $x^*$ represent the initial and optimal model parameters, respectively. }
\item{$\lVert \nabla l_i (x) - \nabla L (x) \rVert \leqslant \varepsilon, \, \forall i,x$; where $\varepsilon$ is the divergence measure.}
\end{enumerate}
\end{assumption}

\noindent Note that the distribution of local datasets in the non-i.i.d fashion breaks the general assumption of $p_i = \frac{m_i}{m}$. Hence, the expectation over clients $\mathbb{E}\{l_i(x)\}$ is not considered equal with the global expectation $\mathbb{E} \{L(x)\}$. The only assumption on relative impact factors is $\sum_{i=1} ^N p_i =1$. 

As the first step through our convergence bounding analysis, we present the following lemma for the local dissimilarity  measure, when having non-identical impacts.

\begin{lemma}[$A$-local dissimilarity]
For the local loss functions $l_i$ with impact factors $p_i$ in the FL global function $L$, there exists $A$ as a measure of dissimilarity at $x$ such that
\begin{equation}
\sum _{i=1} ^{N} p_i \Vert \nabla l_i(x) \Vert \leqslant \Vert \nabla L(x) \Vert  A \quad \forall i,\label{eq:13}
\end{equation}

\begin{proof}
Due to \textbf{Assumption~1}, we have
\begin{equation}
\Vert \nabla l_i(x)-\nabla L(x) \Vert ^2 \leqslant  \varepsilon ^2 
\label{eq:14}
\end{equation}
and
\begin{multline}
\Vert \nabla l_i(x)-\nabla L(x) \Vert ^2 \\
= \Vert \nabla l_i (x) \Vert ^2 - 2 \nabla l_i (x) ^\top \nabla L(x) + \Vert \nabla L(x)\Vert ^2.
\label{eq:15}
\end{multline}
Considering \eqref{eq:15} and multiplying \eqref{eq:14} with $p_i, \forall i$ yields
\begin{multline}
\sum _{i=1} ^{N} p_i \Vert \nabla l_i(x) \Vert ^2 - 2 \sum _{i=1} ^{N} p_i \nabla l_i (x) ^ \top \nabla L (x) \\ + \Vert \nabla L(w) \Vert ^2 \sum _{i=1} ^{N} p_i \leqslant  \varepsilon ^2 \sum _{i=1} ^{N} p_i.
\label{eq:16}
\end{multline}
Considering $\sum _{i=1} ^{N} p_i = 1$ and $\sum _{i=1} ^{N} p_i \nabla l_i (x) = \nabla L (x)$, we have
\begin{multline}
\sum _{i=1} ^{N} p_i \Vert \nabla l_i(x) \Vert ^2  \leqslant  2  \nabla L (x) ^ \top \nabla L (x) - \Vert \nabla L(x) \Vert ^2 + \varepsilon ^2 \\  = \Vert \nabla L(x) \Vert ^2 + \varepsilon ^2 =  \Vert \nabla L(x) \Vert ^2 A_1(x) ^2.
\label{eq:17}
\end {multline}
Note that when $\Vert \nabla L(x) \Vert ^2 \neq 0$, there exists
\begin{equation}
A_1(x) = \sqrt{1 + \frac{\varepsilon ^2}{ \Vert \nabla L(x) \Vert ^2}} \geqslant 1.
\label{eq:18}
\end{equation}

\noindent Therefore, we have
\begin{equation}
\sum _{i=1} ^{N} p_i \Vert \nabla l_i(x) \Vert ^2 \leqslant \Vert \nabla L(x) \Vert ^2 A_1 ^2,\label{eq:19}
\end{equation}
where $A_1$ is the upper bound of $A_1 (x)$.
Considering \eqref{eq:19}, there also exists $A \geqslant 1$ such that
\begin{equation}
\sum _{i=1} ^{N} p_i \Vert \nabla l_i(x) \Vert \leqslant \Vert \nabla L(x) \Vert  A .\label{eq:20}
\end{equation}
This completes the proof.
\end{proof} 
\end{lemma}

Now, the following lemma gives an expected upper bound on the increment of global loss value per-iteration, when DP noise injection is adopted.

\begin{lemma}[Per-iteration expected increment]
The expected difference of global loss functions in two consecutive iterations $(t)$ and $(t+1)$, or the per-iteration expected increment in the value of the loss function, has the following upper limit: 
\begin{multline}
\mathbb{E}\lbrace L(\tilde{x}^{(t+1)}) -  L(\tilde{x}^{(t)})\rbrace \leqslant \lambda _2  \Vert L(\tilde{x}^{(t)})\Vert ^2\\
+ \lambda _1 \mathbb{E} \lbrace \Vert n^ {(t+1)} \Vert \rbrace  \Vert L(\tilde{x}^{(t)})\Vert 
+\lambda _0 \mathbb{E} \lbrace \Vert n^ {(t+1)} \Vert ^2 \rbrace,
\label{eq:21}
\end{multline}

\noindent where
\begin{equation*}
\lambda_2 = -\frac{1}{\mu} +  \frac{A}{\mu} \left( \gamma  + \frac{ \rho (1+ \gamma)}{ \overline{\mu}} \right) + 
\frac{\rho A ^2 {(1+ \gamma)}^2 }{2 {\overline{\mu}}^2},
\end{equation*}
\begin{equation*}
\lambda_1 = 1+ \frac{\rho A (1+ \gamma)}{\overline{\mu}},
\lambda _0 = \frac{\rho}{2},
\end{equation*}
and $n^{(t)}= \sum _{i=1} ^{N} p_i n_i ^{(t)} + n_s ^{(t)}$ is the aggregated noise of the clients and server in each cycle.
\end{lemma}

\begin{proof}
Considering the aggregation process with artificial noises of the client and server side in the $(t+1)$-th aggregation, we have
\begin{equation}
\tilde{x}^{(t+1)}= \sum _{i=1} ^{N} p_i x_i ^{(t+1)} + n ^{(t+1)},\label{eq:22}
\end{equation}
where
\begin{equation}
n^{(t)} = \sum _{i=1} ^{N} p_i n_i ^{(t)} + n_s ^{(t)}
\label{eq:23}
\end{equation}

Because $l_i(\cdot)$ is $\rho$-Lipschitz smooth, we have
\begin{multline}
l_i(\tilde{x}^{(t+1)}) \leqslant l_i(\tilde{x}^{(t)})+
\nabla l_i(\tilde{x}^{(t)})^ \top (\tilde{x}^{(t+1)} - \tilde{x}^{(t)}) \\
+ \frac{\rho}{2} \Vert \tilde{x}^{(t+1)} - \tilde{x}^{(t)}
\Vert ^2
\label{eq:24}
\end{multline}

\noindent for all $\tilde{x}^{(t+1)},~\tilde{x}^{(t)}$. Summation of \eqref{eq:24} multiplied with $p_i, \forall i$ yields
\begin{equation}
\begin{split}
&\sum _{i=1} ^{N} p_i l_i(\tilde{x}^{(t+1)}) \leqslant \sum _{i=1} ^{N} p_i l_i(\tilde{x}^{(t)}) \\ &+
\sum _{i=1} ^{N} p_i \nabla l_i(\tilde{x}^{(t)})^ \top (\tilde{x}^{(t+1)} - \tilde{x}^{(t)}) 
\\ &+ \frac{\rho}{2} \Vert \tilde{x}^{(t+1)} - \tilde{x}^{(t)}\Vert ^2 \sum _{i=1} ^{N} p_i.
\label{eq:25}
\end{split}
\end{equation}
Considering the definition of global loss function $L(\cdot)$ and $\sum _{i=1} ^{N} p_i = 1$, we have
\begin{multline}
L(\tilde{x}^{(t+1)}) -  L(\tilde{x}^{(t)}) \leqslant  
\nabla L(\tilde{x}^{(t)})^ \top (\tilde{x}^{(t+1)} - \tilde{x}^{(t)}) \\
+ \frac{\rho}{2} \Vert \tilde{x}^{(t+1)} - \tilde{x}^{(t)}\Vert ^2
\label{eq:26}
\end{multline}
and therefore,
\begin{equation}
\begin{split}
&\mathbb{E}\lbrace L(\tilde{x}^{(t+1)}) -  L(\tilde{x}^{(t)})\rbrace \leqslant  \\&
\mathbb{E} \left\lbrace \left\langle \nabla L(\tilde{x}^{(t)}), (\tilde{x}^{(t+1)} - \tilde{x}^{(t)})  \right\rangle \right\rbrace 
+ \frac{\rho}{2} \mathbb{E} \lbrace \Vert \tilde{x}^{(t+1)} - \tilde{x}^{(t)}\Vert ^2 \rbrace
\label{eq:27}
\end{split}
\end{equation}

Defining
\begin{equation}
h(x_i ^{(t+1)};\tilde{x} ^{(t)})\triangleq l_i(x_i ^{(t+1)}) +\frac{\mu}{2}\Vert x_i^{(t+1)} - \tilde{x}^{(t)}\Vert ^2,
\label{eq:28}
\end{equation}
we have
\begin{equation}
\nabla h(x_i ^{(t+1)};\tilde{x} ^{(t)})= 
\nabla l_i(x_i ^{(t+1)}) + \mu (x_i^{(t+1)} - \tilde{x}^{(t)}).
\label{eq:29}
\end{equation}
Summation of \eqref{eq:29} multiplied with $p_i,\, \forall i$ yields
\begin{multline}
\sum _{i=1} ^{N} p_i\nabla h(x_i ^{(t+1)};\tilde{x} ^{(t)}) =
\sum _{i=1} ^{N} p_i \nabla l_i(x_i ^{(t+1)})\\ 
+ \mu \sum _{i=1} ^{N} p_i (x_i^{(t+1)} - \tilde{x}^{(t)})=
\sum _{i=1} ^{N} p_i \nabla l_i(x_i ^{(t+1)})\\ 
+ \mu \sum _{i=1} ^{N} p_i x_i^{(t+1)} - \mu \tilde{x}^{(t)}
\label{eq:30}
\end{multline}
and therefore,
\begin{equation}
\begin{split}
&\tilde{x}^{(t+1)} - \tilde{x}^{(t)}=
\sum _{i=1} ^{N} p_i x_i^{(t+1)} + n^{(t+1)} - \tilde{x}^{(t)}\\ &=
\frac{1}{\mu}\left[ \sum _{i=1} ^{N} p_i \left( \nabla h(x_i ^{(t+1)};\tilde{x} ^{(t)}) - \nabla l_i(x_i ^{(t+1)}) \right) \right] +  n^{(t+1)}.
\label{eq:31}
\end{split}
\end{equation}
Substituting \eqref{eq:31} into \eqref{eq:27}, we obtain
\begin{equation}
\begin{split}
&\mathbb{E} \lbrace  L(\tilde{x}^{(t+1)}) -  L(\tilde{x}^{(t)})\rbrace \leqslant   
\mathbb{E}  \Bigg\lbrace  \frac{1}{\mu} \Big\langle \nabla L(\tilde{x}^{(t)}),\\
&\sum _{i=1} ^{N} p_i \left( \nabla h(x_i ^{(t+1)};\tilde{x} ^{(t)} ) - \nabla l_i(x_i ^{(t+1)})\right) \Big\rangle \\ 
&  + \left\langle \nabla L(\tilde{x}^{(t)}), n^{(t+1)}  \right\rangle \Bigg\rbrace
+ \frac{\rho}{2} \mathbb{E} \left\lbrace \Vert \tilde{x}^{(t+1)} - \tilde{x}^{(t)}\Vert ^2 \right\rbrace \\ &=
\mathbb{E}  \Bigg\lbrace  \frac{1}{\mu} \Big\langle \nabla L(\tilde{x}^{(t)}),
\sum _{i=1} ^{N} p_i \left(  \nabla h(x_i ^{(t+1)};\tilde{x} ^{(t)} )  \right. \\ &- \nabla l_i(x_i ^{(t+1)})+ \left. \nabla l_i(\tilde{x} ^{(t)}) \right) - \sum _{i=1} ^{N} p_i \nabla l_i(\tilde{x} ^{(t)}) \Big\rangle \\ 
& + \Big\langle \nabla L(\tilde{x}^{(t)}), n^{(t+1)}  \Big\rangle \Bigg\rbrace + \frac{\rho}{2} \mathbb{E} \left\lbrace \Vert \tilde{x}^{(t+1)} - \tilde{x}^{(t)}\Vert ^2 \right\rbrace 
\\ 
&= -\frac{1}{\mu} \Vert \nabla L(\tilde{x}^{(t)}) \Vert ^2 + \mathbb{E}  \Bigg\lbrace  \frac{1}{\mu} \Big\langle \nabla L(\tilde{x}^{(t)}), \\ &
\sum _{i=1} ^{N} p_i  \nabla h(x_i ^{(t+1)};\tilde{x} ^{(t)} )+ \sum _{i=1} ^{N} p_i \left( \nabla l_i(\tilde{x} ^{(t)})\right. \\& - \left. \nabla l_i(x_i ^{(t+1)}) \right)\Big\rangle \Bigg\rbrace + \mathbb{E}\left\lbrace \Big\langle \nabla L(\tilde{x}^{(t)}), n^{(t+1)}  \Big\rangle \right\rbrace \\&
+ \frac{\rho}{2} \mathbb{E} \left\lbrace \Vert \tilde{x}^{(t+1)} - \tilde{x}^{(t)}\Vert ^2 \right\rbrace
\label{eq:32}
\end{split}
\end{equation}

Now, let us bound $ \Vert \tilde{x}^{(t+1)} - \tilde{x}^{(t)}\Vert$. We know
\begin{equation}
\Vert x_i^{(t+1)} - \tilde{x}^{(t)}\Vert \leqslant
\Vert x_i^{(t+1)} - \hat{x}_i^{(t+1)}\Vert +  
\Vert \hat{x}_i^{(t+1)} - \tilde{x}^{(t)}\Vert,
\label{eq:33}
\end{equation}
where $\hat{x}_i^{(t+1)} =\text{arg} \min _x h_i(x;\tilde{x} ^{(t)})$ Define $\overline{\mu} = \mu - \rho_- > 0$, due to the $\overline{\mu}$-convexity of $h_i(x;\tilde{x} ^{(t)})$ we have
\begin{equation}
\Vert \hat{x}_i^{(t+1)} - x_i^{(t+1)}\Vert \leqslant
\frac{\gamma}{\overline{\mu}} \Vert \nabla l_i(\tilde{x}^{(t)}) \Vert
\label{eq:34}
\end{equation}
and
\begin{equation}
\Vert \hat{x}_i^{(t+1)} - \tilde{x}^{(t)}\Vert \leqslant
\frac{1}{\overline{\mu}} \Vert \nabla l_i(\tilde{x}^{(t)}) \Vert
\label{eq:35}
\end{equation}
where $\gamma \in [0,1]$ denotes a $\gamma$-inexact solution of $\min _x h_i(x;\tilde{x} ^{(t)})$ \cite{ref25}. For such a solution, $x_0$, we have
\begin{equation}
\Vert
\nabla h(x_0; \tilde{x})\Vert \leqslant \gamma \Vert  \nabla h(\tilde{x}; \tilde{x})
\Vert .
\label{eq:36}
\end{equation}
 Now we can use \eqref{eq:34} and \eqref{eq:35} to obtain
\begin{equation}
\Vert x_i^{(t+1)} - \tilde{x}^{(t)}\Vert \leqslant
\frac{1+ \gamma}{\overline{\mu}} \Vert \nabla l_i(\tilde{x}^{(t)}) \Vert.
\label{eq:37}
\end{equation}

\noindent Therefore,
\begin{equation}
\begin{split}
\Vert \tilde{x}^{(t+1)} &- \tilde{x}^{(t)}\Vert
 = \Vert x^{(t+1)} + n ^ {(t+1)} - \tilde{x}^{(t)}\Vert   \\  & \leqslant
\Vert x^{(t+1)} - \tilde{x}^{(t)}\Vert +\Vert n^ {(t+1)} \Vert \\
& = \left\Vert \sum _{i=1} ^ N p_i ( x_i ^{(t+1)} - \tilde{x}^{(t)})\right\Vert +\Vert n^ {(t+1)} \Vert \\
& \leqslant
 \sum _{i=1} ^ N p_i \Vert  x_i ^{(t+1)} - \tilde{x}^{(t)} \Vert +\Vert n^ {(t+1)} \Vert \\
& \leqslant
 \sum _{i=1} ^ N p_i \left( \frac{1+ \gamma}{\overline{\mu}} \Vert \nabla l_i(\tilde{x}^{(t)}) \Vert \right) +\Vert n^ {(t+1)} \Vert \\
& \leqslant
\frac{A (1+ \gamma)}{\overline{\mu}} \Vert \nabla L(\tilde{x}^{(t)}) \Vert +\Vert n^ {(t+1)} \Vert. 
\label{eq:38}
\end{split}
\end{equation}

Since $l_i(\cdot)$ is $\rho$-Lipschitz smooth, we have
\begin{equation}
\Vert
\nabla l_i (\tilde{x}^{(t)}) - \nabla l_i (x_i^{(t+1)})
\Vert
\leqslant \rho 
\Vert
\tilde{x}^{(t)} - x_i^{(t+1)}
\Vert
\label{eq:39}
\end{equation}
Using the triangle inequality, \eqref{eq:36}, \eqref{eq:38}, and \eqref{eq:39}, we obtain
\begin{equation}
\begin{split}
&\Bigg\Vert
\sum _{i=1} ^{N} p_i  \nabla h(w_i ^{(t+1)};\tilde{x} ^{(t)} )+  \sum _{i=1} ^{N} p_i \left( \nabla l_i(\tilde{x} ^{(t)})\right. \\& - \left. \nabla l_i(x_i ^{(t+1)}) \right) 
\Bigg\Vert \leqslant
\left\Vert
\sum _{i=1} ^{N} p_i  \nabla h(x_i ^{(t+1)};\tilde{x} ^{(t)} )\right\Vert \\ & + 
\left\Vert \sum _{i=1} ^{N} p_i \left( \nabla l_i(\tilde{x} ^{(t)})\right. - \left. \nabla l_i(x_i ^{(t+1)}) \right) 
\right\Vert \\ &\leqslant
\sum _{i=1} ^{N} p_i \left\Vert \nabla h(x_i ^{(t+1)};\tilde{x} ^{(t)} )\right\Vert  + 
\sum _{i=1} ^{N} p_i \Big\Vert \left( \nabla l_i(\tilde{x} ^{(t)})\right. \\ &- \left. \nabla l_i(x_i ^{(t+1)}) \right) 
\Big\Vert \leqslant
\gamma \sum _{i=1} ^{N} p_i \Vert \nabla l_i(\tilde{x}^{(t)}) \Vert
\\ & + \rho \sum _{i=1} ^{N} p_i \Vert \tilde{x}^{(t)} - x_i^{(t+1)} \Vert \leqslant  A \gamma \Vert \nabla L(\tilde{x}^{(t)}) \Vert  \\& +\frac{\rho A(1+ \gamma)}{\overline{\mu}} \Vert \nabla L(\tilde{x}^{(t)}) \Vert.
\label{eq:40}
\end{split}
\end{equation}
Then, from \eqref{eq:40} and the Cauchy-Schwarz inequality we have
\begin{equation}
\begin{split}
&\Big\langle \nabla L(\tilde{x}^{(t)}),
\sum _{i=1} ^{N} p_i  \nabla h(x_i ^{(t+1)};\tilde{x} ^{(t)} )
\\ &+ \sum _{i=1} ^{N} p_i \left( \nabla l_i(\tilde{x} ^{(t)})\right.  - \left. \nabla l_i(x_i ^{(t+1)}) \right)\Big\rangle \leqslant
\Vert \nabla L(\tilde{x}^{(t)}) \Vert \\&
\left[ \left(
A \gamma \ +\frac{\rho A (1+ \gamma)}{\overline{\mu}}\right) \Vert \nabla L(\tilde{x}^{(t)}) \Vert \right]
\\& =  \left(
A \gamma \ +\frac{\rho A (1+ \gamma)}{\overline{\mu}}\right) \Vert \nabla L(\tilde{x}^{(t)}) \Vert ^2
\label{eq:41}
\end{split}
\end{equation}

\noindent Substituting \eqref{eq:38} and \eqref{eq:41} into \eqref{eq:32} yields
\begin{equation}
\begin{split}
&\mathbb{E}\lbrace L(\tilde{x}^{(t+1)}) -  L(\tilde{x}^{(t)})\rbrace \leqslant  
-\frac{1}{\mu} \Vert \nabla L(\tilde{x}^{(t)}) \Vert ^2 
\\& + \left(
\frac{A \gamma}{\mu}  +\frac{\rho A (1+ \gamma)}{\mu \overline{\mu}}\right) \Vert \nabla L(\tilde{x}^{(t)}) \Vert ^2 \\
& + \mathbb{E}\lbrace \Vert \nabla L(\tilde{x}^{(t)}) \Vert 
\Vert n^{(t+1)} \Vert \rbrace\\
&+ \frac{\rho}{2} \mathbb{E} \left\lbrace \left[ \frac{A (1+ \gamma)}{\overline{\mu}} \Vert \nabla L(\tilde{x}^{(t)}) \Vert +\Vert n^ {(t+1)} \Vert \right]^2 \right\rbrace.
\label{eq:42}
\end{split}
\end{equation}
Then, we obtain
\begin{multline}
\mathbb{E}\lbrace L(\tilde{x}^{(t+1)}) -  L(\tilde{x}^{(t)})\rbrace \leqslant \lambda _2  \Vert L(\tilde{x}^{(t)})\Vert ^2\\
+ \lambda _1 \mathbb{E} \lbrace \Vert n^ {(t+1)} \Vert \rbrace  \Vert L(\tilde{x}^{(t)})\Vert 
+\lambda _0 \mathbb{E} \lbrace \Vert n^ {(t+1)} \Vert ^2 \rbrace,
\label{eq:43}
\end{multline}
where
\begin{equation*}
\lambda_2 = -\frac{1}{\mu} +  \frac{A}{\mu} \left( \gamma  + \frac{ \rho (1+ \gamma)}{ \overline{\mu}} \right) + 
\frac{\rho A ^2 {(1+ \gamma)}^2 }{2 {\overline{\mu}}^2},
\end{equation*}
\begin{equation*}
\lambda_1 = 1+ \frac{\rho A (1+ \gamma)}{\overline{\mu}} 
\text{ and }
\lambda _0 = \frac{\rho}{2}
\end{equation*}
This completes the proof.
\end{proof}

As expected, lemma $2$ indicates the adverse effect of differential privacy in the expected per-iteration increment of the global loss value. ....

As the final step, we use the per-iteration increment to establish the convergence analysis of the proposed algorithm.
\begin{theorem}[Convergence upper bound of personalized ....]
The upper limit of the difference between the $T$-th and the optimal loss function values defined as the convergence property is given by

\begin{equation}
\begin{split}
& \mathbb{E}\lbrace{L(\tilde{x}^{(T)}) - L(x^{*})}\rbrace \leqslant  \Theta + k_2 T + \frac{k_1 T^{2}}{\epsilon} + \frac{k_0 T^{3}}{\epsilon ^2},
\label{eq:44}
\end{split}
\end{equation}
where $k_{2} = \lambda_{2} \beta^{2},\, k_{1}= \frac{2 \lambda_{1} \beta B c \max \lbrace p_i \rbrace}{\max \lbrace m_i \rbrace} \sqrt{\frac{2N}{\pi}}, \text{ and } k_{0} = \frac{4 \lambda_{0} B^{2} c^{2}{\max \lbrace p_i \rbrace}^2}{{\max \lbrace m_i \rbrace}^2}$.
\end{theorem}

\begin{proof}
Considering the same and independent noise distribution of the additive noise, we define $\mathbb{E}\lbrace \Vert n^{(t)} \Vert \rbrace = \mathbb{E}\lbrace \Vert n \Vert \rbrace  \text{ and }  \mathbb{E}\lbrace \Vert n^{(t)} \Vert ^{2} \rbrace = \mathbb{E}\lbrace \Vert n \Vert ^2 \rbrace$. Applying \eqref{eq:43} recursively for $0 \leqslant t \leqslant T$ yields

\begin{multline}
\mathbb{E}\lbrace{L(\tilde{x}^{(T)}) - L(\tilde{x}^{(0)})}\rbrace \leqslant  T \lambda_{2} \Vert L(\tilde{x}^{(t)})\Vert ^2 \\
+ T \lambda_{1} \Vert L(\tilde{x}^{(t)})\Vert \mathbb{E}\lbrace \Vert n \Vert \rbrace 
+ T \lambda_{0} \mathbb{E}\lbrace \Vert n \Vert ^{2} \rbrace, \label{eq:45}
\end{multline}

\noindent Considering $\Vert L(\tilde{x}^{(t)})\Vert \leqslant \beta$ and  Adding $\mathbb{E} \lbrace L(\tilde{x}^{(0)}) - L(x^{*}) \rbrace$ to both sides of \eqref{eq:45}, we have
\begin{multline}
\mathbb{E}\lbrace{L(\tilde{w}^{(T)}) - L(w^{*})}\rbrace \leqslant \Theta + \lambda_{2} T \beta ^{2} \\
+ \lambda_{1} T \beta \mathbb{E}\lbrace \Vert n \Vert \rbrace 
+ \lambda_{0} T \mathbb{E}\lbrace \Vert n \Vert ^{2} \rbrace , \label{eq:46}
\end{multline}
Since we have $\sigma _A =\frac{\Delta f T c}{\epsilon}$, we obtain
\begin{equation}
\mathbb{E}\lbrace \Vert n \Vert \rbrace = \frac{\Delta f Tc}{\epsilon}\sqrt{\frac{2N}{\pi}} \text{ and } 
\mathbb{E}\lbrace \Vert n \Vert ^2 \rbrace = \frac{\Delta f ^2 T ^2 c ^2 N}{\epsilon ^2}.
\label{eq:47}
\end{equation}

\noindent Setting $\Delta f= 2B\frac{\max \lbrace p_i \rbrace}{\max \lbrace m_i \rbrace}$ and substituting \eqref{eq:47} into \eqref{eq:46}, we have
\begin{equation}
\begin{split}
& \mathbb{E}\lbrace{L(\tilde{x}^{(T)}) - L(x^{*})}\rbrace \leqslant \Theta + \lambda_{2} T \beta ^{2} \\
& +  \frac{2 \lambda_{1} T ^2 \beta B c \max \lbrace p_i \rbrace}{\epsilon \max \lbrace m_i \rbrace} \sqrt{\frac{2N}{\pi}}
+ \frac{4 \lambda_{0} T ^3 B^2 c^2 {\max \lbrace p_i \rbrace} ^2}{\epsilon ^2 {\max \lbrace m_i \rbrace}^2} \\
& = \Theta + k_2 T + \frac{k_1 T^{2}}{\epsilon} + \frac{k_0 T^{3}}{\epsilon ^2},
\label{eq:48}
\end{split}
\end{equation}
where $k_{2} = \lambda_{2} \beta^{2}, k_{1}= \frac{2 \lambda_{1} \beta B c \max \lbrace p_i \rbrace}{\max \lbrace m_i \rbrace} \sqrt{\frac{2N}{\pi}}, \text{ and } k_{0} = \frac{4 \lambda_{0} B^{2} c^{2}{\max \lbrace p_i \rbrace}^2}{{\max \lbrace m_i \rbrace}^2}$.
This completes the proof.
\end{proof}

The last two terms in the right hand side of \eqref{eq:44} depend directly on the amount of  noise. lower $\epsilon$ values strengthen the privacy protection and adversely affect the convergence property. The first two terms, however, are the constant parts depending on the number of iterations. ...

In the above analysis, we saw that by a wise choice of impact factors, $T$, and $N$ we can be confident about the convergence of the FL algorithm while $(\epsilon, \delta)$-DP is used. The number of clients involved in learning in the presented analysis should not necessarily be fixed through training, and this enhances the compatibility of the proposed approach. In the next section, we present the analysis of the same algorithm when impact factors adaptively change throughout the learning process. 

\section{Convergence Analysis of \\ DP in FL with adaptive impact factors}
In this section, we consider an extension to the previous part when impact factors are not fixed during the training. In fact, impacts assigned to clients can vary in each iteration based on the devises' resources or network conditions. The calculated amount of Gaussian noise in section $3$ can still be utilized here, since iterations are independent in noise generation. However, the convergence analysis provided in the previous section needs to be more generalized.

Here, we change $p_i$ to $p_i ^{(t)}$ to represent this adaptability in our equations. Without loss of generality, we assume the relation between two consecutive impact factors to be
\begin{equation}
p_i ^{(t+1)}= p_i ^{(t)} + \alpha _i ^{(t)} \quad \forall i,
\label{eq:49}
\end{equation}
where $\alpha _ i ^{(t)}$ is the amount of change that the relative impact factor assigned to $i$-th client undergoes for $(t+1)$-th iteration. Hence, $\vert \alpha _i \vert \leqslant 1$ and $\sum _{i=1} ^ N \alpha _i ^{(t)}=0 $. 

In order to perform the analysis of the adaptive form, we first present an extension to lemma $2$ and then present the convergence upper bound in theorem $3$.

\begin{lemma}[Per-iteration expected increment: Extension]
The per-iteration expected increment in the value of the loss function, when adaptive $p_i$ is adopted, has the following upper limit: 
\begin{multline}
\mathbb{E}\lbrace L(\tilde{x}^{(t+1)}) -  L(\tilde{x}^{(t)})\rbrace \leqslant \\ \lambda^{\prime} _2  \Vert L(\tilde{x}^{(t)})\Vert ^2
+ \lambda^{\prime} _1 \mathbb{E} \lbrace \Vert n^ {(t+1)} \Vert \rbrace  \Vert L(\tilde{x}^{(t)})\Vert \\
+\lambda^{\prime} _0 \mathbb{E} \lbrace \Vert n^ {(t+1)} \Vert ^2 \rbrace + \frac{1}{2} \max\{l_i\},
\label{eq:50}
\end{multline}

\noindent where
\begin{equation*}
\lambda^{\prime}_2 = -\frac{1}{\mu} +  \frac{A^{\prime}}{\mu} \left( \gamma  + \frac{ \rho (1+ \gamma)}{ \overline{\mu}} \right) + 
\frac{\rho {A^{\prime}} ^2 {(1+ \gamma)}^2 }{2 {\overline{\mu}}^2},
\end{equation*}
\begin{equation*}
\lambda^{\prime}_1 = 1+ \frac{\rho A ^{\prime}(1+ \gamma)}{\overline{\mu}},
\lambda^{\prime} _0 = \frac{\rho}{2},
\end{equation*}
and $n^{(t)}= \sum _{i=1} ^{N} p_i n_i ^{(t)} + n_s ^{(t)}$ is the aggregated noise of the clients and server in each cycle.
\end{lemma}

\begin{proof}
From \eqref{eq:13} we have
\begin{equation}
\sum _{i=1} ^{N} p_i ^{(t)} \Vert \nabla l_i(x ^{(t)}) \Vert \leqslant \Vert \nabla L(x^{(t)}) \Vert  A.
\label{eq:51}
\end{equation}

\noindent Adding $\sum _{i=1} ^{N} \alpha _i ^{(t)} \Vert \nabla l_i (x ^{(t)}) \Vert$ to both sides of \eqref{eq:51} yields
\begin{multline}
\sum _{i=1} ^{N} p_i ^{(t+1)} \Vert \nabla l_i(x ^{(t)}) \Vert \leqslant \\ \sum _{i=1} ^{N} \alpha _i ^{(t)} \Vert \nabla l_i (x ^{(t)}) \Vert +  \Vert \nabla L(x^{(t)}) \Vert  A
\label{eq:52}
\end{multline}
Hence, we have
\begin{equation}
\sum _{i=1} ^{N} p_i ^{(t+1)} \Vert \nabla l_i(x ^{(t)}) \Vert \leqslant \Vert \nabla L(x^{(t)}) \Vert  A^{\prime},
\label{eq:53}
\end{equation}
where
\begin{equation}
A^{\prime} = 
\frac{\sum _{i=1} ^{N} \alpha _i ^{(t)} \Vert \nabla l_i (x ^{(t)}) \Vert}{\Vert \nabla L (x ^{(t)}) \Vert}  + A.
\label{eq:54}
\end{equation}
Therefore, we can bound $\Vert \tilde{x}^{(t+1)} - \tilde{x}^{(t)}\Vert$ as

\begin{equation}
\begin{split}
\Vert \tilde{x}^{(t+1)} &- \tilde{x}^{(t)}\Vert
 = \Vert x^{(t+1)} + n ^ {(t+1)} - \tilde{x}^{(t)}\Vert   \\  & \leqslant
\Vert x^{(t+1)} - \tilde{x}^{(t)}\Vert +\Vert n^ {(t+1)} \Vert \\
& = \left\Vert \sum _{i=1} ^ N p_i ^{(t+1)} ( x_i ^{(t+1)} - \tilde{x}^{(t)})\right\Vert +\Vert n^ {(t+1)} \Vert \\
& \leqslant
 \sum _{i=1} ^ N p_i^{(t+1)} \Vert  x_i ^{(t+1)} - \tilde{x}^{(t)} \Vert +\Vert n^ {(t+1)} \Vert \\
& \leqslant
 \sum _{i=1} ^ N p_i^{(t+1)} \left( \frac{1+ \gamma}{\overline{\mu}} \Vert \nabla l_i(\tilde{x}^{(t)}) \Vert \right) +\Vert n^ {(t+1)} \Vert \\
& \leqslant
\frac{A ^{\prime}(1+ \gamma)}{\overline{\mu}} \Vert \nabla L(\tilde{x}^{(t)}) \Vert +\Vert n^ {(t+1)} \Vert. 
\label{eq:55}
\end{split}
\end{equation}

Summation of \eqref{eq:24} multiplied with $p_i^{(t)}, \forall i$ yields
\begin{equation}
\begin{split}
&\sum _{i=1} ^{N} p_i^{(t)} l_i(\tilde{x}^{(t+1)}) \sum _{i=1} ^{N} p_i^{(t)} l_i(\tilde{x}^{(t)}) \\ &+
\sum _{i=1} ^{N} p_i^{(t)} \nabla l_i(\tilde{x}^{(t)})^ \top (\tilde{x}^{(t+1)} - \tilde{x}^{(t)}) 
\\ &+ \frac{\rho}{2} \Vert \tilde{x}^{(t+1)} - \tilde{x}^{(t)}\Vert ^2 \sum _{i=1} ^{N} p_i ^{(t)}.
\label{eq:56}
\end{split}
\end{equation}

\noindent Considering \eqref{eq:49}, we have
\begin{equation}
\begin{split}
& L(\tilde{x}^{(t+1)}) -  L(\tilde{x}^{(t)})\leqslant  
\sum _{i=1} ^{N} \alpha_i^{(t)} l_i (\tilde{x}^{(t+1)})
\\& +
 \left\langle \nabla L(\tilde{x}^{(t)}), (\tilde{x}^{(t+1)} - \tilde{x}^{(t)})  \right\rangle
+ \frac{\rho}{2} \lbrace \Vert \tilde{x}^{(t+1)} - \tilde{x}^{(t)}\Vert ^2 \rbrace.
\label{eq:57}
\end{split}
\end{equation}

Without loss of generality, we assume $\mathbb{E}\{l_i (x ^{(t)})\} = \frac{1}{N} l_i (x ^{(t)})$, and therefore
\begin{equation}
\begin{split}
\mathbb{E}\left \{ \sum _{i=1} ^N \alpha _i ^{(t)} l_i (x ^{(t)}) \right \} &= \sum _{i=1} ^N \alpha _i ^{(t)} \mathbb{E}\left \{ l_i (x ^{(t)}) \right \} \\ &=
\sum _{i=1} ^N \alpha _i ^{(t)} \left( \frac{1}{N} l_i (x ^{(t)}) \right) \\ & \leqslant \frac{1}{2} \max \{l_i  (x ^{(t)})\}
\end{split}
\label{eq:58}
\end{equation}

\noindent Then, \eqref{eq:57} and \eqref{eq:58} gives
\begin{multline}
\mathbb{E} \left \lbrace L(\tilde{x}^{(t+1)})  -  L(\tilde{x}^{(t)}) \right \rbrace \leqslant \\
\frac{1}{2} \max \{l_i\} +
\mathbb{E} \left\lbrace \left\langle \nabla L(\tilde{x}^{(t)}), (\tilde{x}^{(t+1)} - \tilde{x}^{(t)})  \right\rangle \right\rbrace
\\+
\frac{\rho}{2} \mathbb{E} \left \lbrace \Vert \tilde{x}^{(t+1)} - \tilde{x}^{(t)}\Vert ^2 \right \rbrace.
\label{eq:59}
\end{multline}

Defining $h(\cdot)$ as \eqref{eq:28} and Summation of \eqref{eq:29} multiplied with $p_i^{(t+1)},\, \forall i$ yields 
\begin{multline}
\sum _{i=1} ^{N} p_i^{(t+1)} \nabla h(x_i ^{(t+1)};\tilde{x} ^{(t)}) =
\sum _{i=1} ^{N} p_i ^{(t+1)}\nabla l_i(x_i ^{(t+1)})\\ 
+ \mu \sum _{i=1} ^{N} p_i^{(t+1)} (x_i^{(t+1)} - \tilde{x}^{(t)})=
\sum _{i=1} ^{N} p_i^{(t+1)} \nabla l_i(x_i ^{(t+1)})\\ 
+ \mu \sum _{i=1} ^{N} p_i^{(t+1)} x_i^{(t+1)} - \mu \tilde{x}^{(t)}
\label{eq:60}
\end{multline}

\noindent and therefore,
\begin{equation}
\begin{split}
&\tilde{x}^{(t+1)} - \tilde{x}^{(t)}=
\sum _{i=1} ^{N} p_i ^{(t+1)} x_i^{(t+1)} + n^{(t+1)} - \tilde{x}^{(t)}\\ &=
\frac{1}{\mu}\left[ \sum _{i=1} ^{N} p_i^{(t+1)} \left( \nabla h(x_i ^{(t+1)};\tilde{x} ^{(t)}) - \nabla l_i(x_i ^{(t+1)}) \right) \right]
\\ & +  n^{(t+1)}.
\label{eq:61}
\end{split}
\end{equation}

\noindent Substituting \eqref{eq:61} into \eqref{eq:59}, we obtain
\begin{equation}
\begin{split}
&\mathbb{E} \lbrace  L(\tilde{x}^{(t+1)}) -  L(\tilde{x}^{(t)})\rbrace \leqslant   
\mathbb{E}  \Bigg\lbrace  \frac{1}{\mu} \Big\langle \nabla L(\tilde{x}^{(t)}),\\
&\sum _{i=1} ^{N} p_i^{(t+1)} \left( \nabla h(x_i ^{(t+1)};\tilde{x} ^{(t)} ) - \nabla l_i(x_i ^{(t+1)})\right) \Big\rangle \\ 
&  + \left\langle \nabla L(\tilde{x}^{(t)}), n^{(t+1)}  \right\rangle \Bigg\rbrace
+ \frac{\rho}{2} \mathbb{E} \left\lbrace \Vert \tilde{x}^{(t+1)} - \tilde{x}^{(t)}\Vert ^2 \right\rbrace \\ &
+ \frac{1}{2}\max\{l_i\}
=
\mathbb{E}  \Bigg\lbrace  \frac{1}{\mu} \Big\langle \nabla L(\tilde{x}^{(t)}), \\&
\sum _{i=1} ^{N} p_i^{(t+1)} \left(  \nabla h(x_i ^{(t+1)};\tilde{x} ^{(t)} )  - \nabla l_i(x_i ^{(t+1)}) \right) \\& +
\sum _{i=1} ^{N}  p_i^{(t)} \nabla l_i(\tilde{x} ^{(t)})  - \sum _{i=1} ^{N} p_i^{(t)} \nabla l_i(\tilde{x} ^{(t)}) \Big\rangle \\ 
& + \Big\langle \nabla L(\tilde{x}^{(t)}), n^{(t+1)}  \Big\rangle \Bigg\rbrace + \frac{\rho}{2} \mathbb{E} \left\lbrace \Vert \tilde{x}^{(t+1)} - \tilde{x}^{(t)}\Vert ^2 \right\rbrace 
\\ 
& + \frac{1}{2}\max\{l_i\} = -\frac{1}{\mu} \Vert \nabla L(\tilde{x}^{(t)}) \Vert ^2 + \mathbb{E}  \Bigg\lbrace  \frac{1}{\mu} \Big\langle \nabla L(\tilde{x}^{(t)}), \\ &
\sum _{i=1} ^{N} p_i ^{(t+1)} \nabla h(x_i ^{(t+1)};\tilde{x} ^{(t)} )- \nabla L(x_i ^{(t+1)}) \\& 
+ \nabla L(\tilde{x} ^{(t)}) \Big \rangle \Bigg\rbrace
+ \mathbb{E}\left\lbrace \Big\langle \nabla L(\tilde{x}^{(t)}), n^{(t+1)}  \Big\rangle \right\rbrace \\&
+ \frac{\rho}{2} \mathbb{E} \left\lbrace \Vert \tilde{x}^{(t+1)} - \tilde{x}^{(t)}\Vert ^2 \right\rbrace 
+ \frac{1}{2}\max\{l_i\}
\label{eq:62}
\end{split}
\end{equation}

$\rho$-Lipschitzity of local loss functions leads to have a $\rho$-Lipschitz global loss function. Hence,
\begin{equation}
\Vert
\nabla L (\tilde{x}^{(t)} - \nabla L (x_i ^{(t+1)} )\Vert
\leqslant \rho \Vert \tilde{x}^{(t)} - x_i ^{(t+1)} \Vert 
\label{eq:63}
\end{equation}

\noindent Therefore, using triangle inequality, \eqref{eq:53}, and \eqref{eq:63} we obtain
\begin{equation}
\begin{split}
& \left \Vert \sum _{i=1} ^{N} p_i ^{(t+1)} \nabla h(x_i ^{(t+1)};\tilde{x} ^{(t)} )+ \nabla L(\tilde{x} ^{(t)}) - \nabla L(x_i ^{(t+1)}) \right \Vert
 \\& \leqslant
 \left \Vert \sum _{i=1} ^{N} p_i ^{(t+1)} \nabla h(x_i ^{(t+1)};\tilde{x} ^{(t)}) \right \Vert  + \left \Vert\sum _{i=1} ^{N} p_i ^{(t+1)} \left( \nabla L(\tilde{x} ^{(t)})\right. \right.
 \\& \left. \left. - \nabla L(x_i ^{(t+1)}) \right) \right \Vert 
\leqslant \left( A ^{\prime} \gamma + \frac{\rho A ^{\prime} (1+ \gamma)}{\overline{\mu}} \right) \Vert \nabla L (\tilde{x}^{(t)}) \Vert
\label{eq:64} 
\end{split}
\end{equation}

\noindent Substituting \eqref{eq:64} and \eqref{eq:62} into \eqref{eq:32} yields
\begin{equation}
\begin{split}
&\mathbb{E}\lbrace L(\tilde{x}^{(t+1)}) -  L(\tilde{x}^{(t)})\rbrace \leqslant  
-\frac{1}{\mu} \Vert \nabla L(\tilde{x}^{(t)}) \Vert ^2 
\\& + \left(
\frac{A^{\prime} \gamma}{\mu}  +\frac{\rho A^{\prime} (1+ \gamma)}{\mu \overline{\mu}}\right) \Vert \nabla L(\tilde{x}^{(t)}) \Vert ^2 \\
& + \mathbb{E}\lbrace \Vert \nabla L(\tilde{x}^{(t)}) \Vert 
\Vert n^{(t+1)} \Vert \rbrace\\
&+ \frac{\rho}{2} \mathbb{E} \left\lbrace \left[ \frac{A ^{\prime}(1+ \gamma)}{\overline{\mu}} \Vert \nabla L(\tilde{x}^{(t)}) \Vert +\Vert n^ {(t+1)} \Vert \right]^2 \right\rbrace \\ & + \frac{1}{2}\max\{l_i\}.
\label{eq:65}
\end{split}
\end{equation}

\noindent And, we get
\begin{multline}
\mathbb{E}\lbrace L(\tilde{x}^{(t+1)}) -  L(\tilde{x}^{(t)})\rbrace \leqslant \\ \lambda^{\prime} _2  \Vert L(\tilde{x}^{(t)})\Vert ^2
+ \lambda^{\prime} _1 \mathbb{E} \lbrace \Vert n^ {(t+1)} \Vert \rbrace  \Vert L(\tilde{x}^{(t)})\Vert \\
+\lambda^{\prime} _0 \mathbb{E} \lbrace \Vert n^ {(t+1)} \Vert ^2 \rbrace + \frac{1}{2} \max\{l_i\},
\label{eq:66}
\end{multline}

\noindent where
\begin{equation*}
\lambda^{\prime}_2 = -\frac{1}{\mu} +  \frac{A^{\prime}}{\mu} \left( \gamma  + \frac{ \rho (1+ \gamma)}{ \overline{\mu}} \right) + 
\frac{\rho {A^{\prime}} ^2 {(1+ \gamma)}^2 }{2 {\overline{\mu}}^2},
\end{equation*}
\begin{equation*}
\lambda^{\prime}_1 = 1+ \frac{\rho A ^{\prime}(1+ \gamma)}{\overline{\mu}},
\lambda^{\prime} _0 = \frac{\rho}{2}.
\end{equation*}
This completes the proof.
\end{proof}

\begin{theorem}[Convergence upper bound of adaptive personalized ....]
Using adaptive $p_i$ assignment, the upper limit of the difference between the $T$-th and the optimal loss function values defined as the convergence property is given by
\begin{equation}
\begin{split}
& \mathbb{E}\lbrace{L(\tilde{x}^{(T)}) - L(x^{*})}\rbrace \leqslant  \Theta + k^{\prime}_2 T + \frac{k_1^{\prime} T^{2}}{\epsilon} + \frac{k_0^{\prime} T^{3}}{\epsilon ^2},
\label{eq:67}
\end{split}
\end{equation}

\noindent where $k^{\prime}_{2} = \lambda^{\prime}_{2} \beta^{2} + \frac{\max\{l_i\}}{2} ,\, k^{\prime}_{1}= \frac{2 \lambda^{\prime}_{1} \beta B c \max \lbrace p_i \rbrace}{\max \lbrace m_i \rbrace} \sqrt{\frac{2N}{\pi}}, \text{ and } k^{\prime}_{0} = \frac{4 \lambda^{\prime}_{0} B^{2} c^{2}{\max \lbrace p_i \rbrace}^2}{{\max \lbrace m_i \rbrace}^2}$.
\end{theorem}
\begin{proof}
The proof can be easily extended from the proof for Theorem
$2$ and using lemma $3$.
\end{proof}

\section{Simulation Results}
In this section we evaluate our approach against different privacy budgets and impact factors. We present four scenarios to study the effect of noise and impact factors on convergence bound and accuracy of the models.

\subsection{Experimental Setting}
We evaluate our approach on the real-world Modified National Institute of Standards and Technology (MNIST) dataset \cite{ref36}. MNIST is a widely used dataset for handwritten digit identification which is consisted of 60000 training and 10000 testing samples. We use a multi-layer perceptron (MLP) neural network in local clients and the model weights are communicated with the server for aggregation at each cycle. The designed MLP model classifies the input images using a ReLU activation function in the hidden layer and softmax of 10 classes in the output layer.  To proceed the SGD algorithm in the local optimizers, we set the learning rate equal to $0.02$.

We stablish our evaluation using four scenarios listed in Table~\ref{tab1}. A small randomly chosen subset of the MNIST is distributed between the clients non-identically in each scenario between $60$ clients. The dataset is purposely reduced in size to avoid overfitting. The personalized DP noise is injected in both the client-side and server-side, and the affect of non-identical impact factors during the aggregation process is checked for each scenario. We set $\delta= 0.01$ for the privacy budget, and choose different protection levels ($\epsilon$) throughout this experiment for $30$ global iterations. We further discuss each scenario in the following section.

\begin{table}
\begin{center}
\caption{Simulation Scenarios}
\label{tab1}
\begin{tabular}{| c | c | l |}
\hline
 & number of & \\Scenario &  clients ($N$) & Description \\
\hline
&  & Part $1:$ $20$ clients with severe noisy data\\
1& $60$ & Part $2:$ $20$ clients with moderate noisy data\\
& & Part $3:$ $20$ clients without noise\\
\hline
&  & Part $1:$ $20$ clients with severe noisy data\\
2& $60$ & Part $2:$ $35$ clients with slight noisy data\\
& & Part $3:$ $5$ clients without noise\\ 
\hline
&  & Part $1:$ $20$ clients with $50$ samples\\
3& $60$ & Part $2:$ $20$ clients with $120$ samples\\
& & Part $3:$ $20$ clients with $271$ samples\\ 
\hline 
&  & $t \leqslant 10$: \\
4 & $60$& \, Part $1:$ $20$ clients with severe noisy data\\
&  & \, Part $2:$ $20$ clients with moderate noisy data\\
& & \, Part $3:$ $20$ clients with slight noisy data\\

&  & $10 <t \leqslant 30 $ : \\
&  & \, Part $1:$ $20$ clients with slight noisy data\\
&  & \, Part $2:$ $20$ clients with moderate noisy data\\
&  & \, Part $3:$ $20$ clients with severe noisy data\\

\hline
\end{tabular}
\end{center}
\end{table}

\subsection{Numerical Results}
After distributing the dataset between $60$ clients, we randomly divide clients into three parts. In the following items, the details of each scenario are presented, respectively.
\begin{enumerate}
\item \textit{scenario $1$ }: In the first scenario, we apply the presented privacy protection scheme on clients with heterogeneous data quality. Clients are identical in terms of dataset size, or $m=150$ for all parts, and we deliberately add salt-and-pepper noise with various densities to each part to change data quality between the clients. We first set different impact factors in the non-private mode for comparison. Fig.~\ref{fig:NP1} depicts the importance of impact factors in FL model performances. As it is shown, equal $p_i$ (the green curve) leads to the worst accuracy. The sequence of numbers dedicated to each curve in Fig.~\ref{fig:NP1} represents the relations between impact factors of the three parts. For instance, ``0-1-2" means that we have set $p_i$ equal to $0,\, \frac{1}{N}, \text{and} \frac{2}{N}$ for the first, second, and third part, respectively.

Considering ``0-1-2" as the optimal ratio between the impact factors, Fig.~\ref{fig:DP1} compares the results after applying DP. Here, Gaussian noise is injected using \eqref{eq6} and \eqref{eq11} for protection levels $\epsilon=5$ and $\epsilon=20$ with non-identical impacts. As expected from \eqref{eq:44}, values of the loss function decrease for higher privacy protection levels. In this experiment, we also compare the results when identical $p_i$ is adopted for $\epsilon=20$. As shown in Fig.~\ref{fig:DP1}, the model performance using identical $p_i$ is even worse than a higher protection level $\epsilon=5$, when personalized DP is used.

\begin{figure}[!b]
\begin{center}
\begin{minipage}[c]{1\linewidth}
	\includegraphics[width=\textwidth]{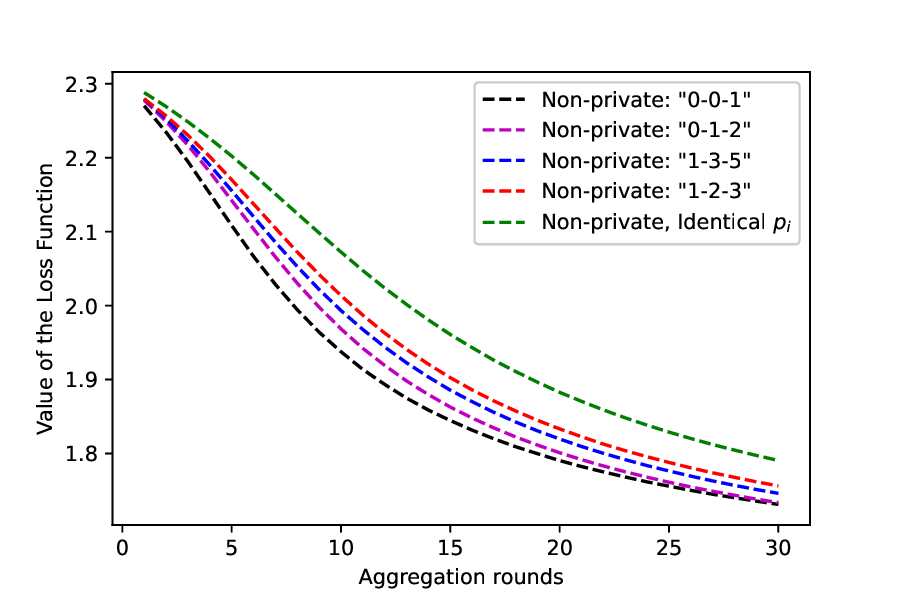}
	\caption{The comparison of loss function values in the non-private mode for five different ways of assigning impact factors to clients of each part in the first scenario.}\label{fig:NP1}
\end{minipage}
\begin{minipage}[c]{1\linewidth}
	\includegraphics[width=\textwidth]{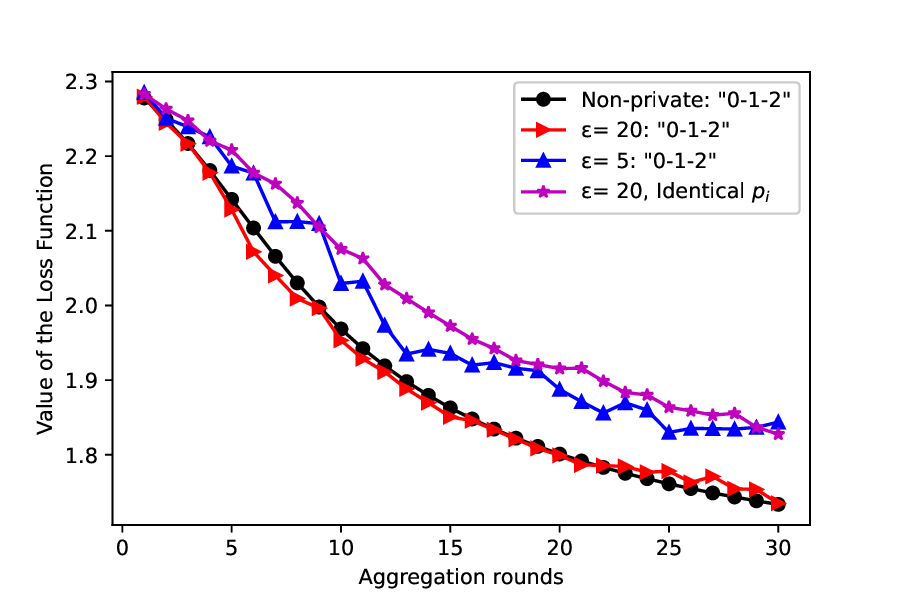}
	\caption{The comparison of loss function values for protection levels $\epsilon=5$ and $\epsilon=20$ when  impact factors of the first, second, and third parts in the first scenario are set as $0$, $\frac{1}{60}$, and $\frac{1}{30}$, respectively. The loss value of the model with identical impacts is also presented for $\epsilon=20$ as a reference. }\label{fig:DP1}
\end{minipage}
\end{center}
\end{figure}

\item \textit{scenario $2$ }: In this scenario, we go further and change clients' distributions in addition to data quality. Hence, we divide clients into three parts of  $20$, $35$, and $5$ clients each, and deliberately add salt-and-pepper noise to them based on densities in Table~\ref{tab1}. Fig.~\ref{fig:NP2} depicts the loss function values in the non-private mode using different impact factor assignments. As shown, equal $p_i$ cannot be a right choice in the presence of heterogeneities. In this case, weighting clients is based on a balance between involving a sufficient number of clients in learning and exploiting the most accurate samples. The red curve related to ``0-2-1''  impact assignment sets impact factors of the first, second and third parts equal to $0$, $\frac{8}{5N}$, and $\frac{4}{5N}$, respectively. 

Considering ``0-2-1" as the basis ratio between the impact factors, model performances in the private mode is compared in Fig.~\ref{fig:DP2}. It is clear from Fig.~\ref{fig:DP2} that a wise choice of $p_i$ significantly improves model accuracy in distributed architectures, especially while using DP.
 
\begin{figure}[!b]
\begin{center}
\begin{minipage}[c]{1\linewidth}
	\includegraphics[width=\textwidth]{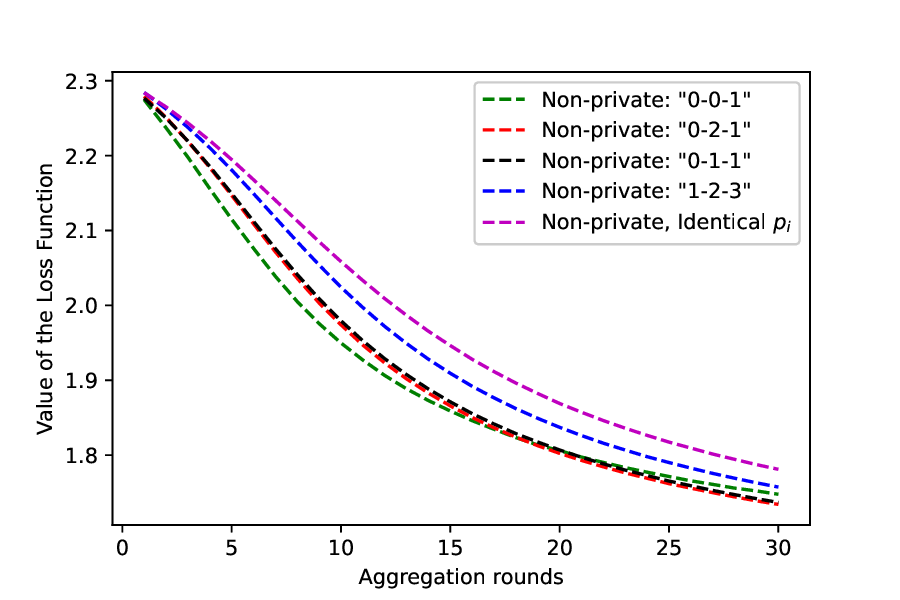}
	\caption{The comparison of loss function values in the non-private mode for five different ways of assigning impact factors to clients of each part in the second scenario.}\label{fig:NP2}
\end{minipage}
\begin{minipage}[c]{1\linewidth}
	\includegraphics[width=\textwidth]{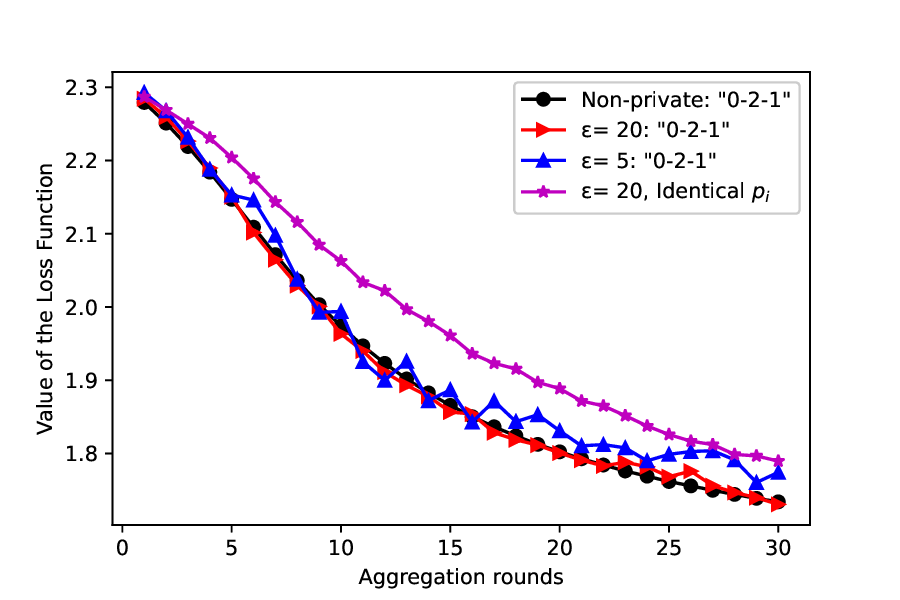}
	\caption{The comparison of loss function values for protection levels $\epsilon=5$ and $\epsilon=20$ when  impact factors of the first, second, and third parts in the second scenario are set as $0$, $\frac{2}{75}$, and $\frac{1}{75}$, respectively. The loss value of the model with identical impacts is also presented for $\epsilon=20$ as a reference. }\label{fig:DP2}
\end{minipage}
\end{center}
\end{figure}

\item \textit{scenario $3$ }: The third scenario is designed to see the effect of the dataset size and impact factors on the convergence performance of FL model. As given in Table~\ref{tab1}, we set $m$ equal to $50$, $120$, and $271$ in clients of part $1$, $2$, and $3$ respectively, and as depicted in Fig.~\ref{fig:NP3}, setting  higher weights to the third part (clients with larger datasets) yields to better results. The common method defining impact factors based on dataset size, or setting $p_i= \frac{m_i}{m}$, is probably developed from this assumption. But it can adversely affect the global model accuracy in heterogeneous structures.

Reducing clients' training samples increases sensitivity and the amount of noise required for preserving DP. Fig.~\ref{fig:DP3} compares the results after Gaussian noise is added for protection levels $\epsilon=5$ and $\epsilon=20$. As it is shown for $\epsilon= 20$, the accuracy of the model when identical $p_i$ is set for all parts is still worse than assigning impact factors proportional to $1$, $3$, and $5$ for part $1$, $2$, and $3$, respectively. This result has been achieved in spite of the fact that more noise is added due to a higher sensitivity in the latter experiment.

\begin{figure}[!b]
\begin{center}
\begin{minipage}[c]{1\linewidth}
	\includegraphics[width=\textwidth]{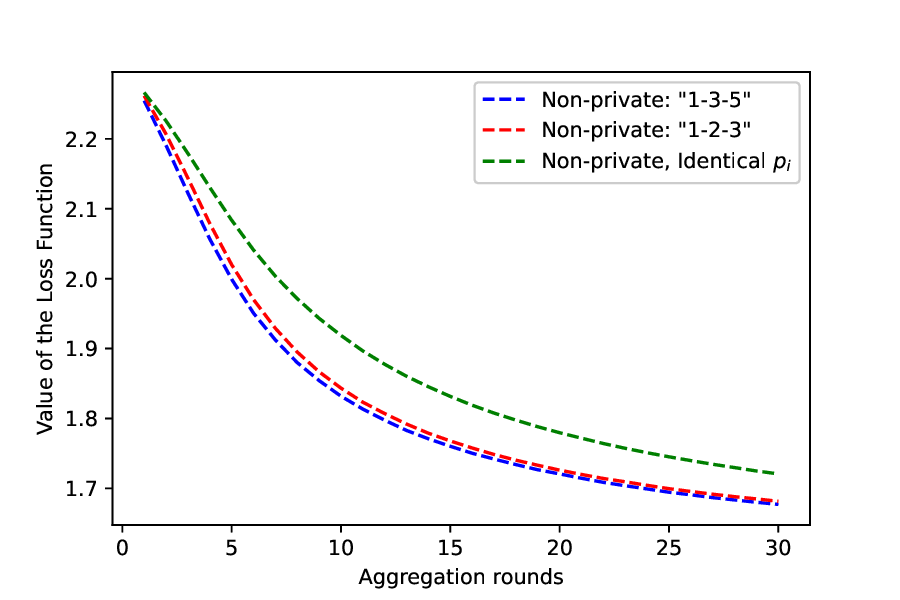}
	\caption{The comparison of loss function values in the non-private mode for three different ways of assigning impact factors to clients of each part in the third scenario.}\label{fig:NP3}
\end{minipage}
\begin{minipage}[c]{1\linewidth}
	\includegraphics[width=\textwidth]{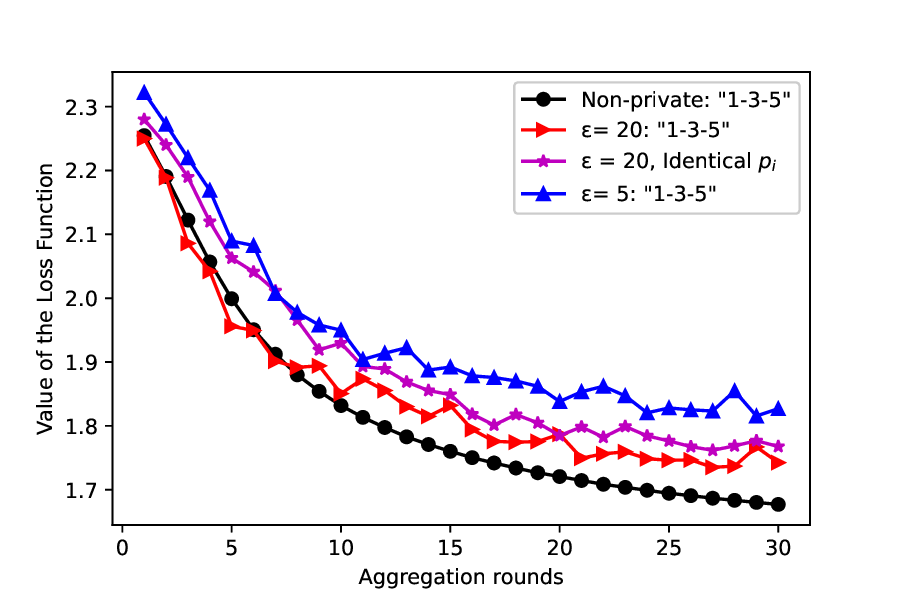}
	\caption{The comparison of loss function values for protection levels $\epsilon=5$ and $\epsilon=20$ when  impact factors of the first, second, and third parts in the third scenario are set as $\frac{1}{180}$, $\frac{1}{60}$, and $\frac{1}{36}$, respectively. The loss value of the model with identical impacts is also presented for $\epsilon=20$ as a reference. }\label{fig:DP3}
\end{minipage}
\end{center}
\end{figure}

\item \textit{scenario $4$ }: We have designed the forth scenario to compare convergence properties when adaptive impact factors are used. As presented in Table~\ref{tab1}, the quality of clients' datasets are changed after the $10$-th aggregation round, and therefore, $p_i$ of each part should be changed to obtain the best model performance. Fig.~\ref{fig:NP4} depicts the results of three types of impact factor assignments. The green curve which belongs to the experiment giving the heaviest weight to slight noisy datasets yields to the fastest and the most accurate result. Considering the green curve as the basis for the private mode, we compare loss function values for protection levels of $\epsilon=5$ and $\epsilon=20$ in Fig.~\ref{fig:DP4}.

\begin{figure}[!b]
\begin{center}
\begin{minipage}[c]{1\linewidth}
	\includegraphics[width=\textwidth]{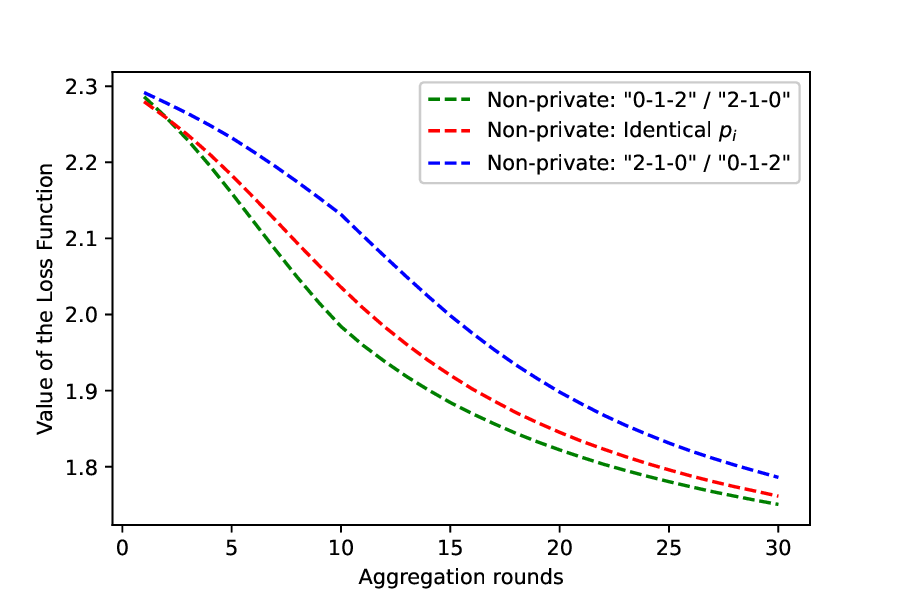}
	\caption{The comparison of loss function values in the non-private mode for three different ways of assigning impact factors to clients of each part in the forth scenario.}\label{fig:NP4}
\end{minipage}
\begin{minipage}[c]{1\linewidth}
	\includegraphics[width=\textwidth]{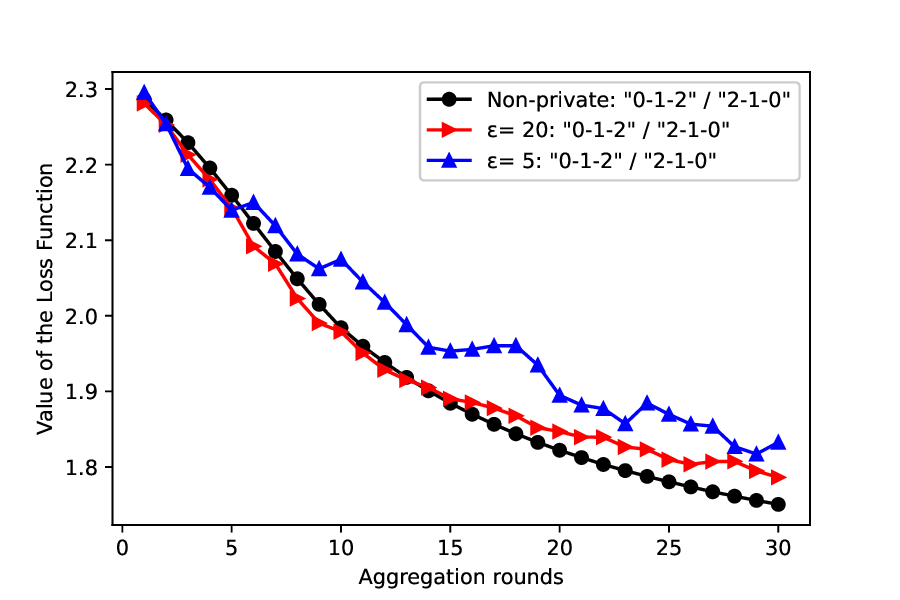}
	\caption{The comparison of loss function values for protection levels $\epsilon=5$ and $\epsilon=20$ when  impact factors of the first, second, and third parts in the forth scenario are respectively set as $0$, $\frac{1}{60}$, and $\frac{1}{30}$ for $t \leqslant 10$, and  $\frac{1}{30}$, $\frac{1}{60}$, and $0$ for $10 < t \leqslant 30$. }\label{fig:DP4}
\end{minipage}
\end{center}
\end{figure}

\end{enumerate}

\section{Conclusion}
In this paper, we presented a personalized privacy preserving approach in federated learning models. Considering the systems and statistical heterogeneities in distributed architectures, we have first focused on the roles that impact factors play in obtaining the best model performance. We further clarified that the impacts are not necessarily fixed during training the global model and undergo changes. Hence, the influence each client has on learning can increase, decrease, or become zero while $\sum_{i=1} ^ N p_i =1$ applies. Then, we have proposed the requirements for preserving $(\epsilon,\delta)$-DP in both clients and the server, when personalized aggregation is applied. We have developed the convergence analysis of the proposed scheme for both fixed and time-varying impact factors. Our simulation results on four scenarios helps understanding the importance of assigning non-identical impact factors to compensate the weaknesses of local datasets, clients, links, and the server.


\bibliographystyle{IEEEtran}
\bibliography{IEEEabrv,mybibfile}

\begin{thebibliography}{10}
\providecommand{\url}[1]{#1}
\csname url@samestyle\endcsname
\providecommand{\newblock}{\relax}
\providecommand{\bibinfo}[2]{#2}
\providecommand{\BIBentrySTDinterwordspacing}{\spaceskip=0pt\relax}
\providecommand{\BIBentryALTinterwordstretchfactor}{4}
\providecommand{\BIBentryALTinterwordspacing}{\spaceskip=\fontdimen2\font plus
\BIBentryALTinterwordstretchfactor\fontdimen3\font minus \fontdimen4\font\relax}
\providecommand{\BIBforeignlanguage}[2]{{%
\expandafter\ifx\csname l@#1\endcsname\relax
\typeout{** WARNING: IEEEtran.bst: No hyphenation pattern has been}%
\typeout{** loaded for the language `#1'. Using the pattern for}%
\typeout{** the default language instead.}%
\else
\language=\csname l@#1\endcsname
\fi
#2}}
\providecommand{\BIBdecl}{\relax}
\BIBdecl

\bibitem{ref1}
B.~McMahan, E.~Moore, D.~Ramage, S.~Hampson, and B.~A. y~Arcas, ``Communication-efficient learning of deep networks from decentralized data,'' in \emph{{AISTATS}}, vol.~54, 2017, pp. 1273--1282.

\bibitem{ref8}
V.~Smith, C.~Chiang, M.~Sanjabi, and A.~S. Talwalkar, ``Federated multi-task learning,'' in \emph{{NIPS}}, 2017, pp. 4424--4434.

\bibitem{ref2}
K.~A. Bonawitz, H.~Eichner, W.~Grieskamp, D.~Huba, A.~Ingerman, V.~Ivanov, C.~Kiddon, J.~Kone{\v{c}}n{\'y}, S.~Mazzocchi, B.~McMahan, T.~V. Overveldt, D.~Petrou, D.~Ramage, and J.~Roselander, ``Towards federated learning at scale: System design,'' in \emph{MLSys}, 2019.

\bibitem{ref3}
Q.~Yang, Y.~Liu, T.~Chen, and Y.~Tong, ``Federated machine learning: Concept and applications,'' \emph{{ACM} Trans. Intell. Syst. Technol.}, vol.~10, no.~2, pp. 12:1--12:19, 2019.

\bibitem{ref4}
J.~Qian, S.~P. Gochhayat, and L.~K. Hansen, ``Distributed active learning strategies on edge computing,'' in \emph{CSCloud/EdgeCom}.\hskip 1em plus 0.5em minus 0.4em\relax {IEEE}, 2019, pp. 221--226.

\bibitem{ref5}
T.~Li, A.~K. Sahu, A.~Talwalkar, and V.~Smith, ``Federated learning: Challenges, methods, and future directions,'' \emph{{IEEE} Signal Process. Mag.}, vol.~37, no.~3, pp. 50--60, 2020.

\bibitem{ref6}
A.~Agarwal and J.~C. Duchi, ``Distributed delayed stochastic optimization,'' in \emph{{CDC}}.\hskip 1em plus 0.5em minus 0.4em\relax {IEEE}, 2012, pp. 5451--5452.

\bibitem{ref7}
K.~Wei, J.~Li, M.~Ding, C.~Ma, H.~H. Yang, F.~Farokhi, S.~Jin, T.~Q.~S. Quek, and H.~V. Poor, ``Federated learning with differential privacy: Algorithms and performance analysis,'' \emph{{IEEE} Trans. Inf. Forensics Secur.}, vol.~15, pp. 3454--3469, 2020.

\bibitem{talaei2024comments}
M.~Talaei and I.~Izadi, ``Comments on ``federated learning with differential privacy: Algorithms and performance analysis",'' \emph{arXiv preprint arXiv:2406.05858}, 2024.

\bibitem{ref9}
B.~McMahan, E.~Moore, D.~Ramage, S.~Hampson, and B.~A. y~Arcas, ``Communication-efficient learning of deep networks from decentralized data,'' in \emph{{AISTATS}}, ser. Proceedings of Machine Learning Research, vol.~54.\hskip 1em plus 0.5em minus 0.4em\relax {PMLR}, 2017, pp. 1273--1282.

\bibitem{ref10}
S.~Salehkaleybar, A.~Sharif{-}Nassab, and S.~J. Golestani, ``One-shot federated learning: Theoretical limits and algorithms to achieve them,'' \emph{J. Mach. Learn. Res.}, vol.~22, pp. 189:1--189:47, 2021.

\bibitem{ref11}
Y.~Du, S.~Yang, and K.~Huang, ``High-dimensional stochastic gradient quantization for communication-efficient edge learning,'' \emph{{IEEE} Trans. Signal Process.}, vol.~68, pp. 2128--2142, 2020.

\bibitem{ref12}
N.~Shlezinger, M.~Chen, Y.~C. Eldar, H.~V. Poor, and S.~Cui, ``Uveqfed: Universal vector quantization for federated learning,'' \emph{{IEEE} Trans. Signal Process.}, vol.~69, pp. 500--514, 2021.

\bibitem{ref15}
A.~Reisizadeh, A.~Mokhtari, H.~Hassani, A.~Jadbabaie, and R.~Pedarsani, ``Fedpaq: {A} communication-efficient federated learning method with periodic averaging and quantization,'' in \emph{{AISTATS}}, ser. Proceedings of Machine Learning Research, vol. 108.\hskip 1em plus 0.5em minus 0.4em\relax {PMLR}, 2020, pp. 2021--2031.

\bibitem{ref13}
H.~Wang, S.~Sievert, S.~Liu, Z.~B. Charles, D.~S. Papailiopoulos, and S.~J. Wright, ``{ATOMO:} communication-efficient learning via atomic sparsification,'' in \emph{NeurIPS}, 2018, pp. 9872--9883.

\bibitem{ref14}
N.~Strom, ``Scalable distributed {DNN} training using commodity {GPU} cloud computing,'' in \emph{{INTERSPEECH}}.\hskip 1em plus 0.5em minus 0.4em\relax {ISCA}, 2015, pp. 1488--1492.

\bibitem{ref16}
T.~Nishio and R.~Yonetani, ``Client selection for federated learning with heterogeneous resources in mobile edge,'' in \emph{{ICC}}.\hskip 1em plus 0.5em minus 0.4em\relax {IEEE}, 2019, pp. 1--7.

\bibitem{ref17}
L.~Ye and V.~Gupta, ``Client scheduling for federated learning over wireless networks: {A} submodular optimization approach,'' in \emph{{CDC}}.\hskip 1em plus 0.5em minus 0.4em\relax {IEEE}, 2021, pp. 63--68.

\bibitem{ref18}
A.~Ghosh, J.~Hong, D.~Yin, and K.~Ramchandran, ``Robust federated learning in a heterogeneous environment,'' \emph{CoRR}, vol. abs/1906.06629, 2019.

\bibitem{ref19}
B.~Recht, C.~R{\'{e}}, S.~J. Wright, and F.~Niu, ``Hogwild: {A} lock-free approach to parallelizing stochastic gradient descent,'' in \emph{{NIPS}}, 2011, pp. 693--701.

\bibitem{ref26}
T.~Li, S.~Hu, A.~Beirami, and V.~Smith, ``Federated multi-task learning for competing constraints,'' \emph{CoRR}, vol. abs/2012.04221, 2020.

\bibitem{ref28}
X.~Li, K.~Huang, W.~Yang, S.~Wang, and Z.~Zhang, ``On the convergence of fedavg on non-iid data,'' in \emph{{ICLR}}.\hskip 1em plus 0.5em minus 0.4em\relax OpenReview.net, 2020.

\bibitem{ref27}
T.~Li, A.~K. Sahu, M.~Zaheer, M.~Sanjabi, A.~Talwalkar, and V.~Smith, ``Federated optimization in heterogeneous networks,'' \emph{Proceedings of Machine Learning and Systems}, vol.~2, pp. 429--450, 2020.

\bibitem{ref29}
V.~Mothukuri, R.~M. Parizi, S.~Pouriyeh, Y.~Huang, A.~Dehghantanha, and G.~Srivastava, ``A survey on security and privacy of federated learning,'' \emph{Future Gener. Comput. Syst.}, vol. 115, pp. 619--640, 2021.

\bibitem{ref30}
N.~Carlini, C.~Liu, {\'{U}}.~Erlingsson, J.~Kos, and D.~Song, ``The secret sharer: Evaluating and testing unintended memorization in neural networks,'' in \emph{{USENIX} Security Symposium}.\hskip 1em plus 0.5em minus 0.4em\relax {USENIX} Association, 2019, pp. 267--284.

\bibitem{ref31}
M.~Nasr, R.~Shokri, and A.~Houmansadr, ``Comprehensive privacy analysis of deep learning: Stand-alone and federated learning under passive and active white-box inference attacks,'' \emph{CoRR}, vol. abs/1812.00910, 2018.

\bibitem{ref20}
C.~Dwork and A.~Roth, ``The algorithmic foundations of differential privacy,'' \emph{Found. Trends Theor. Comput. Sci.}, vol.~9, no. 3-4, pp. 211--407, 2014.

\bibitem{talaei2024adaptive}
M.~Talaei and I.~Izadi, ``Adaptive differential privacy in federated learning: A priority-based approach,'' \emph{arXiv preprint arXiv:2401.02453}, 2024.

\bibitem{ref23}
K.~Wei, J.~Li, M.~Ding, C.~Ma, H.~H. Yang, F.~Farokhi, S.~Jin, T.~Q.~S. Quek, and H.~V. Poor, ``Federated learning with differential privacy: Algorithms and performance analysis,'' \emph{{IEEE} Trans. Inf. Forensics Secur.}, vol.~15, pp. 3454--3469, 2020.

\bibitem{ref34}
O.~Thakkar, G.~Andrew, and H.~B. McMahan, ``Differentially private learning with adaptive clipping,'' \emph{CoRR}, vol. abs/1905.03871, 2019.

\bibitem{ref32}
X.~Liu, H.~Li, G.~Xu, R.~Lu, and M.~He, ``Adaptive privacy-preserving federated learning,'' \emph{Peer-to-Peer Netw. Appl.}, vol.~13, no.~6, pp. 2356--2366, 2020.

\bibitem{ref33}
M.~Gong, K.~Pan, Y.~Xie, A.~K. Qin, and Z.~Tang, ``Preserving differential privacy in deep neural networks with relevance-based adaptive noise imposition,'' \emph{Neural Networks}, vol. 125, pp. 131--141, 2020.

\bibitem{ref35}
R.~Hu, Y.~Guo, H.~Li, Q.~Pei, and Y.~Gong, ``Personalized federated learning with differential privacy,'' \emph{{IEEE} Internet Things J.}, vol.~7, no.~10, pp. 9530--9539, 2020.

\bibitem{ref21}
B.~Recht, C.~R{\'{e}}, S.~J. Wright, and F.~Niu, ``Hogwild: {A} lock-free approach to parallelizing stochastic gradient descent,'' in \emph{{NIPS}}, 2011, pp. 693--701.

\bibitem{ref22}
B.~McMahan, E.~Moore, D.~Ramage, S.~Hampson, and B.~A. y~Arcas, ``Communication-efficient learning of deep networks from decentralized data,'' in \emph{{AISTATS}}, ser. Proceedings of Machine Learning Research, vol.~54.\hskip 1em plus 0.5em minus 0.4em\relax {PMLR}, 2017, pp. 1273--1282.

\bibitem{ref24}
M.~Chen, D.~G{\"{u}}nd{\"{u}}z, K.~Huang, W.~Saad, M.~Bennis, A.~V. Feljan, and H.~V. Poor, ``Distributed learning in wireless networks: Recent progress and future challenges,'' \emph{{IEEE} J. Sel. Areas Commun.}, vol.~39, no.~12, pp. 3579--3605, 2021.

\bibitem{ref25}
T.~Li, A.~K. Sahu, M.~Zaheer, M.~Sanjabi, A.~Talwalkar, and V.~Smith, ``Federated optimization in heterogeneous networks,'' in \emph{MLSys}.\hskip 1em plus 0.5em minus 0.4em\relax mlsys.org, 2020.

\bibitem{ref36}
Y.~LeCun, L.~Bottou, Y.~Bengio, and P.~Haffner, ``Gradient-based learning applied to document recognition,'' \emph{Proc. {IEEE}}, vol.~86, no.~11, pp. 2278--2324, 1998.

\end{thebibliography}

\newpage

\vspace{11pt}

\vspace{11pt}


\vfill

\end{document}